%% file: main.tex
\documentclass[letterpaper,11pt]{article}
\usepackage[utf8]{inputenc}
\usepackage[margin=1in]{geometry}

\usepackage{geometry} 
\usepackage{graphicx} 

\usepackage{booktabs} 
\usepackage{array} 
\usepackage{paralist} 
\usepackage{verbatim} 
\usepackage{subfig} 
\input{header}

\usepackage[nottoc,notlof,notlot]{tocbibind} 
\usepackage[titles,subfigure]{tocloft} 

\newcommand{\eps}{\epsilon}

\usepackage{hyperref}
\usepackage{cleveref}

\renewcommand{\cref}[1]{\Cref{#1}}
\title{On Communication Complexity of Classification Problems}
\author{
Daniel M. Kane\thanks{Department of Computer Science and Engineering/Department of Mathematics, University of California, San Diego. {\tt dakane@ucsd.edu.}}
\and Roi Livni\thanks{Princeton University, Department of Computer Science. {\tt rlivni@cs.princeton.edu.}}
\and Shay Moran\thanks{Institute for Advanced Study, Princeton. {\tt  shaymoran1@gmail.com.}}
\and  Amir Yehudayoff\thanks{The Technion University, Israel. {\tt amir.yehudayoff@gmail.com.}}
}

\begin{document}
\maketitle

\input{abstract}
\newpage
\tableofcontents

\input{introduction}
\input{preliminaries}
\input{techniques}

\input{discussion}

\bibliographystyle{alpha}
\bibliography{ref}

\appendix

\input{proofs-decision}

\input{proofs-convex}

\input{appendix}
\input{apndx-improper}

\end{document}

%% file: header.tex
\usepackage{amsthm,amsmath,amsfonts,color}
\usepackage{framed}
\usepackage{tcolorbox,amssymb}
\usepackage{thmtools}
\usepackage{thm-restate}
\usepackage[nottoc,notlof,notlot]{tocbibind} 
\usepackage[titles,subfigure]{tocloft} 
\usepackage{hyperref}
\usepackage{cleveref}
\newcommand{\ignore}[1]{}
\DeclareMathOperator*{\E}{\mathbb{E}}
\newtheorem{definition}{Definition}
\declaretheorem[name=Theorem]{theorem}
\newtheorem*{theorem*}{Theorem}
\newtheorem{lemma}{Lemma}
\newtheorem*{lemma*}{Lemma}

\newtheorem{obs}{Observation}
\newtheorem{example}{Example}
\newtheorem{claim}{Claim}[lemma]
\newcommand{\X}{\mathcal{X}}
\newcommand{\Z}{\mathcal{Z}}
\renewcommand{\H}{\mathcal{H}}
\newcommand{\N}{\mathbb{N}}
\newcommand{\alice}{Alice~}

\newcommand{\conp}{\mathrm{coNP}}
\newcommand{\np}{\mathrm{NP}}
\newcommand{\p}{\mathrm{P}}
\newcommand{\vc}{\textrm{VC-dim}}
\newcommand{\covc}{\textrm{coVC-dim}}

\newcommand{\err}[1]{1[#1(x)\ne y]}
\newcommand{\T}{D_{\H}}
\newcommand{\Tnp}{N^{np}_{\H}}
\newcommand{\Tc}{N^{conp}_{\H}}
\newcommand{\Tp}{\T}
\newcommand{\R}{\mathbb{R}}

\newcommand{\xx}{\mathbf{x}}

\newcommand{\sbd}[1]{I_{\{#1\}}}
\newcommand{\air}[1]{A^{(#1)}_{m,r-1}}
\newcommand{\bir}[1]{B^{(#1)}_{m,r-1}}

\renewcommand{\paragraph}[1]{~\\\noindent\textbf{#1}}
\newcommand{\concatxx}[2]{\left<{#1};{#2}\right>}
\newcommand{\concatx}[2]{\concatxx{S_{#1}}{S_{#2}}}
\newcommand{\concat}{\concatx{a}{b}}

\let\oldaddcontentsline\addcontentsline
\newcommand{\stoptocentries}{\renewcommand{\addcontentsline}[3]{}}
\newcommand{\starttocentries}{\let\addcontentsline\oldaddcontentsline}

%% file: abstract.tex
\abstract{
This work {studies distributed learning
in the spirit of Yao's model of communication complexity}:
consider a two-party setting,
where each of the players gets a list of labelled examples
and they communicate in order to jointly perform some learning task.
To naturally fit into the framework of learning theory,
the players can send each other examples (as well as bits) where each example/bit costs one unit of communication.  
{This enables a uniform treatment of infinite classes
such as half-spaces in $\R^d$, which are ubiquitous in machine learning. 
}

We study several fundamental questions in this model. 
{For example, we provide combinatorial characterizations
of the classes that can be learned with efficient communication in the proper-case
as well as in the improper-case.
These findings imply unconditional separations between various 
learning contexts, e.g.\ realizable versus agnostic learning,
proper versus improper learning, etc.

The derivation of these results hinges on a type of decision problems
we term ``{\it realizability problems}'' where the goal is
deciding whether a distributed input sample is consistent with an hypothesis from a pre-specified class.

%
%
%

From a technical perspective, 
the protocols we use are based on ideas from machine learning theory and the impossibility results
are based on ideas from communication complexity theory.}

%% file: introduction.tex

\newpage
\section{Introduction}

Communication complexity provides 
a basic and convenient framework 
for analyzing the information flow in computational systems~\cite{Yao79}.
As such, it has found applications in various areas ranging 
from distributed systems, where communication 
is obviously a fundamental resource,
to seemingly disparate areas like data structures, 
game theory, linear programming, extension complexity of polytopes, and many others 
(see e.g.~\cite{Kushilevitz97book} and references within).

We {study} a {\em distributed learning} variant of this model,
where two learners in separate locations
wish to jointly solve some learning problem.
We consider communication protocols in which each of the two parties, 
Alice and Bob, receives a sequence of labelled examples as input, 
and their goal is to perform some learning task;
for example, to agree on a function 
with a small misclassification rate,
or even just to decide whether such a function exists in some pre-specified class.
{To enable the treatment of general learning problems
where the input domain is not necessarily finite
(e.g.\ half-spaces in $\R^d$)},
a transmission of a single input example is considered
as an atomic unit of communication
(``standard'' transmission of bits is also allowed).
The ability to send examples empowers the protocols,
and makes proving lower bounds more challenging.}

We consider several distributed learning problems corresponding to agnostic/realizable learning
by proper/improper protocols. 
Two representative examples of our results concern realizable-case learnability:
\begin{itemize}
\item {\bf Proper learning:} \cref{thm:pchar} and \cref{thm:proper} 
provide a combinatorial characterization\footnote{This characterization holds under mild assumptions on the class $\H$ - see the discussion after \cref{thm:proper}.} of properly learnable classes:
a class $\H$ is properly learnable if and only if 
it has a finite {\it VC dimension} and a finite {\it coVC dimension} 
(a new combinatorial parameter that rises in this context).
\item {\bf Improper learning:} It has been previously shown that any class with a finite VC dimension
can be efficiently learned with polylogarithmic communication using a variant of Adaboost \cite{Daume12efficient,Balcan12dist}.
\cref{{thm:reallb}} provides a matching lower bound. This lower bound already applies for the class of half-planes.
\end{itemize}
We further get (unconditional) separations
between realizable-case versus agnostic-case learning, and proper
versus non-proper learning  (see~\cref{thm:properlb} and~\cref{thm:agnlb}).

To derive these results  we introduce a type of decision problems called \emph{realizability problems}. 
In the realizability problem for a class $\H$,
the parties need to decide whether there is a function in $\H$  
that is consistent with their input  examples.
We define the analogs of the complexity classes P, NP, and coNP in this context;
for example, $\p$ contains all classes 
for which the realizability problem can be solved with polylogarithmic communication complexity.
We derive some basic results such as  $\p = \np \cap \conp$.
An analogous statement holds in classical communication complexity~\cite{Aho83notions},
however our result does not seem to follow from it;
in particular, our proof uses different techniques,
including boosting and metric properties of VC classes.

Another aspect of realizability problems is that different instances
of them include communication problems that may be interesting in their own right.
For example, consider the following geometric variant
of  the \emph{set disjointness} problem:
each of Alice and Bob gets $n$
points in the plane and their goal is to decide whether
their convex hulls are disjoint.
This is an instance of the realizability problem
for the class of half-planes.
Unlike the classical {set disjointness} problem,
this variant can be efficiently solved. 
In fact, we give almost matching upper and lower
bound\footnote{In this work,
we write $\tilde{O}, \tilde{\Omega}, \tilde{\Theta}$ to hide logarithmic factors; for example,
$f(n) \leq \tilde{O}(g(n))$ if there are constants~$\alpha,\beta>0$
so that $f(n) \leq \alpha g(n) \log (g(n)) + \beta$ for all $n$.}
of $\tilde\Theta(\log n)$ transmitted points 
for solving this planar convex set disjointness problem.
We also consider {the convex set disjointness problem} in $\R^d$ for general $d$, 
but our bounds are not tight in terms of $d$.
\footnote{A variant of this problem in which
the inputs are taken from a fixed finite set $X$,
and the goal is to decide whether the convex hulls
intersect in a point from $X$ was studied by~\cite{lovasz93communication}.
This variant turns to be much harder;
for example, if $X$ is a set of $n$ points on the unit circle 
then it is equivalent to the standard set disjointness problem.}



%

\paragraph{{Related work.}}
The communication model considered in this paper 
can be seen as a blend between Yao's model~\cite{Yao79} and Abelson's model~\cite{Abelson80}.
Similar to Abelson's model, 
in which the parties receive a list of real numbers as input
and transmit them for a unit communication cost\footnote{In fact, Abelson's model allows to send any smooth function of the inputs for a unit cost.}, 
also here transmission \emph{examples},
which may come from an infinite domain, is considered as an atomic unit of communication. 
{In contrast with Abelson's setting, 
which focuses on computing differentiable functions,
the focus here is on combinatorial problems that are closer in spirit to the problems studied in Yao's model.

Algorithmic study of distributed learning has seen vast amount of research 
(a partial list includes~\cite{AD11, DGSX12, AS15, SST14, SS14}).
{Similar models to the one considered in this paper as well
as extensions (e.g.\ to multiparty communication protocols) appeared in these works.}
Some works even explicitly phrased upper bounds on the number of transmitted examples~\cite{Daume12efficient,Daume12protocols,Balcan12dist,Chen16boosting}.
{However, all previous communication lower bounds were in terms of the number of bits and not the number examples.}


This setting can also be thought of as an interactive/distributed variant
of {\em sample compression schemes}:
{sample compression schemes}
are a well-studied notion in learning theory that was introduced in~\cite{Littlestone86relating}.
Sample compression schemes can be seen as
protocols in which only one party (say Alice) gets an input sample
that is consistent with some known-in-advanced hypothesis class, 
and her goal is to transmit as few examples as possible to Bob
so he may reconstruct an hypothesis that is consistent
with Alice's input (including the examples she did not send him).

{Another related work by~\cite{Huang14epsapprox} concerns the communication complexity
of computing {\it $\eps$-approximations}: for example, consider a set $S$ of $N$
points in the plane that is distributed between the parties whose goal is
to communicate as few points as possible in order to find for every axis-aligned rectangle $R$,
its relative density $\frac{\lvert S\cap R\rvert}{N}$, up to $\eps$-error. 
We use similar ideas to derive non-trivial upper bounds in the agnostic setting
(see \cref{thm:agnub}).}

%% file: preliminaries.tex
\section{Model and Main Results}\label{sec:modelandresults}
\subsection{Communication Model}\label{sec:model}
We follow standard notation from machine learning (see e.g.\ the book~\cite{shalev14understanding}).
Let $\X$ be a domain, and let $\Z=\X\times\{\pm 1\}$ be the examples domain.
We denote by $\Z^*=\bigcup_{n}\Z^n$ the set of all samples.
For a sample $S \in \Z^n$, we call $n$ the {\em size} of $S$,
and denote it by $|S|$.

We study communication protocols between two parties called \alice and Bob.
Each party receives a sample as an input. 
Alice's input is denoted by $S_a$ and
Bob's input by $S_b$. 
Let $S=\concat$ denote the joint sample
that is obtained by concatenating Alice's and Bob's samples.
Similarly to other works in distributed learning
we do not assume an underlying distribution on samples
(for a related discussion see~\cite{Daume12protocols} and references within).
Specifically, the sample $S$ can be adversarially distributed between Alice and Bob. 




\paragraph{Communication Protocols.}
We focus on deterministic protocols which we define next.
Following~\cite{Yao79}, a protocol
$\Pi$  is modelled by a rooted directed tree.
Each internal node $v$ is owned by 
exactly one of the parties and each outgoing edge
from $v$ corresponds to an example in $\Z$ in a one-to-one and onto fashion
(so each internal node has out-degree $\lvert\Z\rvert$).
Each internal node $v$ is further associated with a function $f_v:\Z^*\to \Z$
such that $f_v(S')\in S'$ for every $S'\in \Z^*$.
\emph{The value $f_v(S')$ is interpreted as the example that is communicated
on input $S'$ when the protocol reached state $v$.
The restriction $f_v(S')\in S'$ amounts to that during the protocol 
each party may only send examples from her/his input sample}.

\begin{description}
\item[Execution:]
Every pair of inputs $S_a,S_b$ induces a selection 
of a unique outgoing edge for every internal node:
if $v$ is owned by Alice then select the edge labelled by $f_v(S_a)$,
and similarly for Bob.
This in turn defines a unique path from the root to a leaf.
\item[Output:] The leafs of the protocol are labelled by its outputs.
Thus, the output of the protocol on input $S_a,S_b$
is the label of the leaf on the path corresponding to $S_a,S_b$.
\item[Complexity:]
Let $T:\N\to\N$, we say that $\Pi$ has \emph{sample complexity} at most $T$
if the length of the path corresponding to an input sample $S=\concat$
is at most $T\bigl(\lvert S\rvert\bigr)$.
\end{description}

\medskip
\noindent{\em \bf Transmission of bits.}
We will often use {\em hybrid} protocols
in which the parties also send bits to each other.
While we did not explicitly include this possibility 
in the above definition,
it can still be simulated\footnote{At the beginning of the protocol,
each of Alice and Bob publishes two examples,
say $z_a,z_a'$ of Alice and $z_b,z_b'$ of Bob.
Then, Alice encodes the 4 messages $z_a,z_a',0,1$
using a prefix free code of length 2 on the letters $z_a,z_a'$,
and similarly Bob.
Now, they can simulate any protocol that also uses bits
with only a constant blow-up in the sample complexity.
} within the defined framework.


\paragraph{{Uniformity.}}
{The possibility of sending a whole example as an atomic unit of communication
(which, if $\Z$ is infinite, requires infinitely many bits to encode)  
makes this communication model \emph{uniform} in the sense
that it enables solving communication problems with an infinite input-domain by a \emph{single} protocol.
For example, consider the convex set disjointness problem that is discussed in \cref{sec:overviewdec};
in this problem each Alice and Bob gets as inputs $n$ points from $\R^d$.

This deviation from Yao's model, which is non-uniform\footnote{Yao's model is non-uniform in the sense 
that each communication problem is in fact a sequence of problems parametrized by the input size.},
may take some time to get used to; 
however, we would like to point out that it is in fact natural in 
the context of machine learning where infinite hypothesis classes
(such as half-spaces, neural-nets, polynomial threshold functions,...) are common.
\cref{thm:pcomp} further addresses the connection between
the communication complexity of the decision problems we consider here
and the communication complexity of (a variant of) these problems 
in Yao's (non-uniform) model.}

The problems we study can be naturally partitioned into \emph{search problems}, and \emph{decision problems}.
We next describe our main results, following this partition.

%

\subsection{Search/learning Problems}\label{sec:search}
The search problems we consider are concerned with finding an {optimal} (or near optimal) hypothesis $h$ with respect to a given hypothesis class $\H$. 
The objective will generally be to minimize the number of mistakes $h$ performs on the input samples
$S = \concat$;
that is, to minimize
\begin{align}\label{eq:error} L_S(h)=\frac{1}{\lvert S\rvert}\sum_{(x,y)\in S} 1[h(x)\neq y].\end{align}

{Note that Alice can easily find a hypothesis
$h_a \in \H$ that minimizes $L_{S_a}(h_a)$, and similarly for Bob.
The difficultly in solving the problem stems from that
none of the players knows all of $S$,
and their goal is to find $h$ with minimum $L_S(h)$
with least communication.}

A class $\H$ is said to be \emph{learnable} with sample complexity~$T=T(\eps,n)$
{if for every $\eps$} there is a protocol\footnote{One could also define that $\eps$ is given as a common input to both Alice and Bob.} 
that for every input sample $S=\concat$ of size $n$
transmits at most $T(\eps,n)$ examples and 
outputs an hypothesis $h$ with
\[L_S(h)\leq \min_{f\in \H}L_S(f) + \eps .\]

A running example the reader may keep in mind is when $\H$
is the class of half-planes in~$\R^2$ .

We distinguish between the \emph{realizable} and the \emph{agnostic} cases,
and between \emph{proper} and \emph{improper} protocols. 
We refer to the case when there exists $h\in \H$ with $L_S(h)=0$
as the realizable case (in contrast to the agnostic case), 
and to protocols whose output $h$ always belongs to $\H$ as proper protocols
(in contrast to the {improper protocols} that may output~$h\notin \H$).




\subsubsection*{Search Problems Main Results}

\paragraph{Realizable case.}
{We begin with a lower bound for improper learning in the realizable setting. 
In terms of upper bounds, several authors provided clever implementations of Boosting algorithms 
in various distributed settings \cite{Balcan12dist, Daume12efficient, freund95}.
Specifically, \cite{Balcan12dist, Daume12efficient} used Adaboost to learn a class $\H$ 
with sample complexity $O(d\log 1/\eps)$, where $d$ is the VC dimension of $\H$. 
The algorithm developed in \cite{Balcan12dist} can be easily adapted\footnote{The result in \cite{Balcan12dist} assumes a center that can mediate interaction.} 
to the communication model considered here
(for completeness this is done in \cref{apx:improper}).}
Our first main result shows tightness of the aforementioned upper bound in terms of $d$ and $\epsilon$ separately.  
\begin{restatable}[Realizable case - lower bound]{theorem}{reallb}\label{thm:reallb}
Let $\H$ be the class of half-spaces in $\R^d$, $d\ge 2$, {and $\eps\leq 1/3$}.
Then, any protocol that  learns $\H$ in the realizable case has sample complexity
at least $\tilde\Omega(d+\log(1/\eps))$.
\end{restatable}
{The lower bound in terms of $d$ holds for every class $\H$ (not necessarily half-spaces), 
and follows from standard generalization bounds derived from sample compression schemes
{(A slightly weaker lower bound of $\Omega(d)$ bits appears in \cite{Balcan12dist})}.
The dependence in $\eps$ though, may be significantly smaller for different classes $\H$. 
For example, the class of thresholds over $\R$ admits a  protocol 
that outputs a consistent hypothesis (namely $\eps=0$) with sample complexity $O(1)$.

{We derive the lower bound in terms of $\epsilon$} by proving a more general} 
trade-off between  the number of rounds and the sample complexity. 
A more detailed examination of our proof yields the following round-communication tradeoff:
every protocol that learns the class of half-planes 
with $\eps$ error using at most $r$ rounds 
must have sample complexity at least
\[\tilde\Omega\Bigl(\frac{(1/\eps)^{1/r}}{\log (1/\eps)} + r\Bigr).\]
This matches an upper bound given by Theorem~10 in~\cite{Balcan12dist}.

{The proof of \cref{thm:reallb}, which appears in \cref{prf:reallb}, follows from 
a lower bound on the realizability decision problem for half-planes. 
The main challenge is in dealing with protocols that learn the class of half-planes in an improper manner;
i.e.\ their output is not necessarily a halfplane (indeed, the general boosting-based protocols of
\cite{Balcan12dist,Daume12efficient} are improper)}.
The idea, in a nutshell, is to {consider} a \emph{promise} variant of the realizability problem in which Alice and Bob
just need to distinguish whether the input sample is realizable or \emph{noisy}
(a sample $S$ is noisy if there is $x\in\R^2$ such that both $(x,1),(x,-1)$
are in $S$). 
{In \cref{sec:cvxoverview} we outline these arguments in more detail.}

\paragraph{Proper learning in the realizable case.}
In the realizable and proper case ({namely the input sample is still realizable, but the protocol must output a hypothesis in the class}), 
an exponentially larger lower bound holds:
\begin{restatable}[Realizable \& proper case - lower bound]{theorem} {properlb}\label{thm:properlb}
There exists a class $\H$ with VC dimension $1$ 
such that every protocol that learns $\H$ properly has sample complexity of at least~$\tilde\Omega(1/\eps)$.
Moreover, this holds even if the input sample is realizable.
\end{restatable}
The proof, which appears in \cref{prf:properlb},
 implies an exponential separation between proper and improper sample complexities in the realizable case.
The proof of \cref{thm:properlb} exhibits a VC dimension 1 class for which it is impossible to decide whether the input sample $S$
is realizable, unless $\tilde\Omega(\lvert S\rvert)$ examples are transmitted.
{This shows that in some cases improper learning is strictly easier
than proper learning (the boosting-based protocol of~\cite{Balcan12dist, Daume12efficient}
gives an upper bound of $\tilde O (\log 1/\eps)$).}

\paragraph{A  characterization of properly learnable classes.} 
Given the $\tilde\Omega(1/\eps)$ lower bound for proper learning in the realizable setting, 
it is natural to ask which classes $\H$ can be properly learned with non-trivial sample complexity. 
{We provide a combinatorial characterization of these classes
in \cref{thm:pchar} and \cref{thm:proper} that uses the VC dimension and a new combinatorial parameter we term the {\it coVC dimension}.}
{The same theorem also provides a complexity-theoretical characterization: 
every class $\H$ is properly learnable with logarithmic sample complexity}
if and only if the corresponding realizability problem can be solved efficiently.

\paragraph{Agnostic case.}
{We now move to the agnostic case. 
Namely, the input sample is no longer assumed to be realizable, 
and the protocol (which is not assumed to be proper) 
needs to output a hypothesis 
with error that is larger by at most $\eps$ than the error of the best $h\in\H$.}
\begin{restatable}[Agnostic case - lower bound]{theorem}{agnlb}\label{thm:agnlb}
There exists a hypothesis class of VC dimension~$1$ 
such that every protocol that learns $\H$ in the agnostic case
has sample complexity of at least $\tilde\Omega\left(1/\eps\right)$.
\end{restatable}
The proof appears in \cref{prf:agnlb}.
This theorem, together with the upper bounds in \cite{Balcan12dist, Daume12efficient}, 
implies an exponential separation {(in terms of $\eps$)} between sample complexities
in the realizable case and the agnostic case.
In fact, the class of VC dimension~$1$ used in the proof is the class of singleton over $\mathbb{N}$.
This particular class can be learned in the realizable case using just $O(1)$ examples:
if any of the parties get a $1$-labelled example then he/she publishes it,
and they output the corresponding
singleton; and otherwise they output the function which is constantly $-1$.


We also observe that in the agnostic case there is a non-trivial upper bound:
\begin{restatable}[Agnostic case - upper bound]{theorem}{agnub}\label{thm:agnub}
Every class $\H$ is learnable in the agnostic case
with sample complexity $\tilde O_d\bigl((1/\eps)^{2-\frac{2}{d+1}} + \log n\bigr)$
where $d$ is the VC dimension of $\H$,
and $\tilde O_d(\cdot)$ hides a constant that depends on $d$.
\end{restatable}
{We remark that this is achieved using a proper protocol.}
The proof, which is given in \cref{prf:agnub}, is based on the notion of \emph{$\eps$-approximation} and uses
a result due to~\cite{Matousek93discrepancy}.
The above bound beats the trivial $O(d/\eps^2 + \log n)$ upper bound which follows 
from the statistical agnostic sample complexity
and can be derived as follows\footnote{{For simplicity we describe a randomized protocol. 
A corresponding deterministic protocol can be derived similarly like in the proof of~\cref{thm:agnub}.}}:
Alice and Bob sample $O(d/\eps^2)$ examples from $S=\concat$ and output $h\in H$
with minimal error on the published examples.
The extra $\log n$ bits come from exchanging the sizes $\lvert S_a\rvert, \lvert S_b\rvert$,
in order to sample uniformly from $S$.
We do not know whether the dependence on $\eps$ is tight,
even for the basic class of half-planes.

A relevant remark is that if one relaxes the learning requirement
by allowing the output hypothesis $h$ a slack of the form
\[L_S(h)\leq c\cdot\min_{f\in \H}L_S(f) + \eps,\]
where $c$ is a universal constant then the logarithmic dependence on $1/\eps$ from the realizable case
can be restored; The work of \cite{Chen16boosting} implies that for every $c>4$
such a protocol exists with sample complexity $O(\frac{d\log(1/\eps)}{c-4})$.
We do not know what is the general correct tradeoff between $c, \eps,d$ 
in this case (see discussion in \cref{sec:openquestions}).



\subsection{Decision Problems}{\label{sec:decision}}

A natural decision problem in the context of distributed learning is 
the {\em realizability} problem.
In this problem, Alice and Bob are given
input samples $S_a,S_b$ and they need to decide whether
there exists $h\in \H$ such that $L_S(h)=0$ where $S=\concat$.

As a benchmark example, consider the case where $\H$ is the class of half-spaces. 
If we further assume that Alice receives positively labelled points 
and Bob receives negatively labelled points, then the problem becomes 
the \emph{convex set disjointness} problem where Alice and Bob need to decide 
if the convex hulls of their inputs intersect.

\paragraph{Complexity Classes.}
With analogy to communication complexity theory in Yao's model,
we define the complexity classes $\p,\np$, and $\conp$ for realizability problems.
Roughly speaking, the class $\H$ is in $\p$ if there is an efficient protocol 
(in terms of sample complexity) for the realizability problem over $\H$,
it is in $\np$ if there is a short proof that certifies realizability,
and it is in $\conp$ if there is a short proof that certifies non realizability.

Let $T$ denote an $\N\to \N$ function.
We say that $\H$ has \emph{sample complexity} at most $T$,
and write $\Tp(n) \leq T(n)$, if there exists a protocol with sample complexity at most $T(n)$
that decides the realizability problem for $\H$,
where $n$ is the size of the input samples.

\newcommand{\true}{\mathsf{True}}
\newcommand{\false}{\mathsf{False}}

\begin{definition}[The class $\p$]
The class $\H$ is in $\p$ if $\Tp(n) \leq \mathrm{poly}(\log n)$.
\end{definition}
 We say the $\H$ has \emph{non-deterministic} sample complexity at most $T$,
and write $\Tnp(n) \leq T(n)$, if there exist predicates $A,B:\Z^*\times\Z^*\to\{\true,\false\}$ such that:
\begin{enumerate}
\item For every realizable sample $S=\concat\in\Z^n$ there exists
a \emph{proof} $P \in S^{T(n)}$ such that $A(S_a,P)=B(S_b,P)=\true$.
\item For every non realizable sample $S=\concat\in\Z^n$ and for every
{proof} $P \in \Z^{T(n)}$ either $A(S_a,P)= \false$ or $B(S_b,P)=\false$.
\end{enumerate}
Intuitively, this means that if $S$ is realizable then there is a subsample
of it of length $T(n)$ that proves it,
but if it is not then no sample of size $T(n)$
can prove it.

\begin{definition}[The class $\np$]
The class $\H$ is in $\np$ if 
$\Tnp(n) \leq \mathrm{poly}(\log n)$. 
\end{definition}

The \emph{co-non-deterministic} sample complexity $\Tc$ of $\H$
is defined similarly, interchanging the roles of realizable and non-realizable samples. 
Unlike the typical relation between NP and coNP,
where the co-non-deterministic complexity of a function $f$ 
is the non-deterministic complexity of another function, namely $\lnot f$,
in this setting $\Tc$ is \emph{not} $N^{np}_{\H'}$
of another {class~$\H'$.}

\begin{definition}[The class $\conp$]
The class $\H$ is in $\conp$ if 
$\Tc(n) \leq \mathrm{poly}(\log n)$. 
\end{definition}

\paragraph{VC and coVC Dimensions.} 
We next define combinatorial notions that 
(almost) characterize the complexity classes defined above.

Recall that the \emph{VC dimension} of $\H$ is the size of the 
largest set $R\subseteq \X$ that is \emph{shattered} by $\H$;
namely every sample $S$ with distinct sample points in $R$ is realizable by $\H$. 
As we will later see, every class that is in $\np$
has a bounded VC dimension.

We next introduce a complementary notion, 
which will turn out to fully characterize~coNP.
\begin{definition}[coVC dimension]
The \emph{coVC dimension} of $\H$
is the smallest integer $k$ such that 
every non realizable sample
has a non realizable subsample of size at most $k$.
\end{definition}
A non realizable subsample serves as a proof for non realizability.
Thus small coVC dimension implies small $\conp$ sample complexity.
It turns out that the converse also holds (see \cref{thm:covconp}). 

The VC and coVC dimensions are, in general, uncomparable.
Indeed, there are classes with VC dimension $1$ and arbitrarily large coVC dimension
and vice versa. An example of the first type is the class of singletons
over $[n]=\{1,\ldots,n\}$; its VC dimension is $1$ and its coVC dimension is $n$
as witnessed by the sample that is constantly $-1$.
An example of the second type is the class $\{h:[n]\to\{{\pm 1}\} : \forall i\geq n/2 \ \ h(i)={-1}\}$
that has VC dimension $n/2$ and coVC dimension $1$; any non realizable sample must contain
an example $(i,1)$ with $i\geq n/2$, which is not realizable.

For the class of half-spaces in $\R^d$, both dimensions are roughly the same
(up to constant factors). It is a known fact that its VC dimension is $d+1$, 
and for the coVC dimension we have:
\begin{example}\label{example:covchyp}
The coVC dimension of the class of half-spaces in $\R^d$ is at most~$2d+2$.
\end{example}
This follows directly from Carath{\'e}odory's theorem. Indeed, let $S$ be a non realizable sample and denote by $S_{+}$ the positively labelled set and $S_{-}$ the negatively labelled set. 
Since $S$ is not realizable, the convex hulls of $S_{+},S_{-}$ intersect. 
Let  $x$ be a point in the intersection. 
By Carath{\'e}odory's theorem, $x$ lies in the convex hull 
of some $d+1$ positive points and in the convex hull of some $d+1$ negative points.
Joining these points together gives a non realizable sample of size $2d+2$. 

%
\subsubsection*{Decision Problems Main Results}


Our first main result characterizes 
the class $\p$ in terms of the VC and coVC dimensions, 
and shows that $\p=\np\cap\conp$ in this context. 

\begin{restatable}[A Characterization of $\p$]{theorem}{pchar}\label{thm:pchar}
The following statements are equivalent for a hypothesis class $\H$:
\begin{enumerate}[(i)]
\item\label{it:p} $\H$ is in $\p$.
\item\label{it:npconp} $\H$ is in $\np\cap\conp$.
\item\label{it:vccovc} $\H$ has a finite VC dimension and a finite coVC dimension.
\item\label{it:dk} There exists a protocol for the realizability problem for $\H$ 
with sample complexity $\tilde O(dk^2 \log |S|)$ where $d = \vc(\H)$ and $k=\covc(\H)$.
\end{enumerate}
\end{restatable}
The proof and the protocol
in \cref{it:dk} appear in \cref{alg:main}. 
The proof of the theorem reveals an interesting dichotomy:
the sample complexity of the realizability problem over $\H$
is either $O(\log n)$ or at least $\tilde \Omega (n)$;
there are no problems
of ``intermediate'' complexity. 
\ignore{This phenomenon is related to the ``uniformity'' of the model;
the same protocol should work for all input sizes
(Yao's model is non-uniform in this sense).}

The theorem specifically implies that for every $d$ half-spaces in $\R^d$ are in $\p$
since both the VC and coVC dimensions are $O(d)$.
It also implies as a corollary that the convex set disjointness problem
can be decided by sending at most $\tilde O(d^3\log n)$ points.

The proof of \cref{thm:pchar} is divided to two parts.
One part shows that if the VC and coVC dimensions are large then
the $\np$ and $\conp$ complexities are high as well 
(see \cref{thm:vcnp} and \cref{thm:covconp} below).
The other part shows that if both the VC and coVC dimensions are small
then the realizability problem can be decided efficiently.
This involves a carefully tailored variant of Adaboost,
which we outline in \cref{sec:overviewdec}.

The equivalence between the first two items shows that $\p = \np\cap\conp$.
This means that whenever there are short certificates for the realizability and non realizability
then there is also an efficient protocol that decides it.
An analogous equivalence in Yao's model was established  by~\cite{Aho83notions}. 
We {compare these results} in more detail in \cref{sec:npconp}.

\ignore{How sharp is the bound in \cref{it:dk} {in \cref{thm:pchar}}?}
The following two theorems
give lower bounds on the sample complexity 
in terms of $\vc$ and $\covc$.

\begin{restatable}[``$\vc \leq \np$'']{theorem}{vcnp}\label{thm:vcnp}
For every class $\H$
with VC dimension $d\in{\mathbb{N}\cup\{\infty\}}$,
{$$\Tnp(n)=\tilde\Omega(\min({d,n})).$$}
\end{restatable}

\begin{restatable}[``$\covc = \conp$'']{theorem}{covcconp}\label{thm:covconp}
For every class $\H$
with coVC dimension $k\in{\mathbb{N}\cup\{\infty\}}$,
{$$\Tc(n)=\tilde \Theta(\min(k,n)).$$}
\end{restatable}
The proofs of the theorems appear in \cref{prf:vcnp} and \cref{prf:covconp}.
\cref{thm:covconp} gives a characterization of $\conp$
in terms of $\covc$, while \cref{thm:vcnp} only gives
one of the directions.
It remains open whether the other direction also holds for $\np$. 

The next theorem shows that also the $\log |S|$ dependence in \cref{thm:pchar} 
is necessary.
Specifically,  it is necessary for the class of half-planes.
We note that both the $\np$ and the $\conp$ sample complexities of this class are constants (at most 4).
%


\begin{restatable}[Realizability problem -- lower bound]{theorem}{logn}\label{thm:logn}
Any protocol that decides the realizability problem 
for the class of half-planes in $\R^2$ must have sample complexity at least~$\tilde\Omega(\log n)$
for samples of size $n$.
\end{restatable}

{\cref{thm:logn} is implied by \cref{thm:lognstrong},
which is a stronger result that we discuss in  \cref{prf:logn}. 
\cref{thm:lognstrong} concerns a promise variant of the realizability problem 
and also plays a crucial role in the derivation of \cref{thm:reallb}. 
In \cref{sec:cvxoverview} we overview the arguments used in the derivation of these results.}


We next state a \emph{compactness} result that enables
transforming bounds from Yao's model to {the one considered here} and vice versa.
A natural approach of studying the realizability problem in Yao's model 
is by ``discretizing'' the domain; more specifically, fix a finite 
set $R\subseteq \X$, and consider the realizability problem
with respect to restricted class $\H|_R = \{h|_R : h\in H\}$.
This restricted problem is well defined in Yao's model, 
since every example $(x,y)$ can be encoded using at most $1+\lceil\log \lvert R\rvert\rceil$ bits.
Using this approach, one can study the realizability
problem with respect to the bigger class $\H$,
in a non-uniform way, by taking into consideration
the dependence on $\lvert R\rvert$.
As the next result shows, the class $\p$ 
remains invariant under this alternative approach: 
\begin{restatable}[Compactness for $\p$]{theorem}{pcomp}\label{thm:pcomp}
Let $\H$ be a hypothesis class over a domain $\X$.
Then, the following statements are equivalent.

\begin{enumerate}[(i)]
\item\label{it:1c} $\H$ is  in $\p$.
\item\label{it:2c} For every finite $R\subseteq \X$ there is a protocol that decides the realizability problem for $\H|_R$
with sample complexity at most $c\cdot \log(n)$
for inputs of size $n$, where $c$ is a constant depending only on $\H$.
\item\label{it:3c} For every finite $R\subseteq \X$ there is an efficient protocol that decides the realizability problem for $\H|_R$
in Yao's model with bit complexity at most $c\cdot \log^m \lvert R\rvert$, where $c$ and $m$ are constants depending only on $\H$.
\end{enumerate}
\end{restatable}
The proof of \cref{thm:pcomp} appears in \cref{prf:pcomp}.
A similar result holds for $\conp$ | this follows from Theorem~\ref{thm:covconp}.
We do not know whether such a result holds for~$\np$.

\paragraph{{Proper learning.}}
{We next address the connection between the realizability problem | the task of deciding whether the input is consistent with $H$, and proper learning in the realizable case | the task of finding a consistent $h\in H$.
For this we introduce the following definition.  
A class $\H$ is \emph{closed}\footnote{This is consistent with the topological notion of a closed set:
if one endows $\{\pm 1\}^X$ with the product topology then this definition agrees with
$\H\subseteq\{\pm 1\}^X$ being closed in the topological sense.} if for every $h\notin \H$ there is a finite sample $S$ that is consistent with $h$ and is not realizable by $\H$.}
{Note that every class $\H$ can be extended to a class $\bar \H$
that is closed and has the same VC and coVC dimensions
(by adding to $\bar \H$ all $h\notin \H$ that viloate the requirement). 
The classes $\bar \H$ and $\H$ are indistinguishable with respect to any finite sample.
That is, a finite sample $S$ is realizable by $\H$ if and only if it is realizable by~$\bar \H$.}

\begin{restatable}[Proper learning -- characterization]{theorem}{proper}\label{thm:proper}
{Let $\H$ be a closed class with $d = \vc(\H)$ and $k = \covc(\H)$.
If $\H \in \p$ then 
it is properly learnable in the realizable case 
with sample complexity $\tilde{O}(d k^2 \log(1/\epsilon))$. 
If $\H$ is not in $\p$, then the sample complexity for properly learning $\H$ in the realizable setting is at least $\tilde\Omega(1/\epsilon)$. }
\end{restatable}

{We do not know whether \cref{thm:proper} can be extended to non-closed classes. However, it does extend under other natural restrictions.
For example, it applies when $X$ is countable, even when $\H$ is not closed.
For a more detailed discussion see \cref{prf:proper}.}

%% file: techniques.tex
\section{Proof Techniques Overview}\label{sec:proofoverview}
In this section we give a high level overview of the main proof techniques. 
{

In \cref{sec:overviewdec} we give a simplified version of the protocol for the realizability problem, 
which is a main ingredient in the proof of \cref{thm:pchar}. 
This simplified version solves the special instance of convex set disjointness, 
and highlights the ideas  used in the general case.

In \cref{sec:overviewset} we briefly overview the set disjointness problem from Yao's model,
which serves as a tool for deriving lower bounds.
For example, it is used in the bound for agnostic learning in \cref{thm:agnlb}
and the bound for proper learning in \cref{thm:properlb}.
Set disjointness also plays a central role in relating the non-deterministic complexities with the VC and coVC dimensions
(\cref{thm:vcnp} and \cref{thm:covconp}).

In \cref{sec:cvxoverview} we overview the construction of the hard instances 
for the convex set disjointness problem,
which is used in the lower bounds for the realizability problem and in \cref{thm:logn}. 
We also discuss the implication of the construction for the lower bound for learning
in the realizable case (\cref{thm:reallb}).
}
\subsection{A Protocol for the Realizability Problem}
\label{sec:overviewdec}

To get a flavor of the arguments used in
the proof of \cref{thm:pchar},
we exhibit a protocol for convex set disjointness,
which is a special instance of the $\np\cap\conp\subseteq \p$ direction.
Recall that in the convex set disjointness problem,
Alice and Bob get as inputs two sets $X,Y\subseteq\R^d$
of size $n$, and they need to decide 
whether the convex hulls of $X,Y$  intersect.

\begin{figure}
\begin{tcolorbox}\label{alg:conv}
\begin{center}
{\bf Protocol for convex set disjointness}\\
\end{center}
\textbf{Input:} Let $X,Y\subset \R^d$ denote Alice's and Bob's inputs. \\ \ \\
\textbf{Protocol:}
\begin{itemize}
\item Let $\eps=\frac{1}{100d}$ and $n = |X|+|Y|$.
\item Alice sets $W_0(x)=1$ for each $x\in X$.
\item For $t=1,\ldots, T=2(d+1)\log n$
\begin{enumerate}
\item Alice sends Bob an $\eps$-net $N_i\subseteq X$ with
respect to the distribution $p_t(x) = \frac{W_{t-1}(x)}{\sum_x{W_{t-1}(x)}}$. 
\item Bob checks whether the convex hulls of $Y$ and $N_t$ have a common point.
\item If they do, Bob reports it and outputs INTERSECTION. 
\item Else, Bob sends Alice the $d+1$ support vectors from $N_t\cup Y$ that
encode a hyperplane $h_t$ that separates $Y$ from $N_t$.
\item Alice sets $W_{t}(x) = W_{t-1}(x)/2$ if $x$
is separated from $Y$ by $h_t$ and $W_{t}(x)=W_{t-1}(x)$
otherwise.
\end{enumerate}
\item Output DISJOINT.
\end{itemize}
\end{tcolorbox}
\caption{A $\tilde O(d^3\log n)$ sample complexity protocol for convex set disjointness}\label{fig:algconv}
\end{figure}

We can think of the protocol as simulating a boosting algorithm
(see \cref{fig:algconv}).
It proceeds in $T$ rounds, where at round $t$, 
Alice maintains a probability distribution $p_t$ on $X$,
and {requests} a weak hypothesis for it. 
Bob serves as a weak learner and {provides Alice} a weak hypothesis $h_t$
for $p_t$. 

The first obstacle is that to naively simulate this protocol, 
Alice would need to transmit $p_t$, which is a probability distribution, 
and Bob would need to transmit $h_t$, which is an hypothesis, 
and it is not clear how to achieve this with efficient communication complexity.

Our solution is as follows. 
At each round~$t$, Alice draws an \emph{$\frac{1}{100d}$--net} with respect to $p_t$ and transmits it to Bob.
Here, an $\eps$-net is a subset $N_t$ of Alice's points satisfying that every halfspace that contains an $\eps$ fraction of Alice's points with respect to~$p_t$ must contain a point in $N_t$.
Bob in turn checks whether the convex hulls of $N_t$ and~$Y$ intersect.
If they do then clearly the convex hulls of $X,Y$ intersect and we are done.
Otherwise, Bob sends Alice a hyperplane $h_t$ that separates $Y$ from $N_t$.
One way of sending $h_t$ is by the $d+1$ \emph{support vectors}\footnote{This is a subset of $N_t\cup Y$ that encodes a separator of $N_t$ and $Y$ with maximal margin. Note that formally, Bob cannot send points from $N_t$ however, since Alice already sent $N_t$ and so this can be handled using additional bits of communication. }.
The crucial point is that, by the $\eps$-net property, 
this hyperplane separates $Y$
from a $1-\frac{1}{100d}$ fraction of $X$ with respect to $p_t$.
%

Why does this protocol succeeds?  
The interesting case is when the protocol continues successfully for all of the $T$ iterations. 
{The challenge is to show that in this case the convex hulls must be disjoint.}
The main observation that allows us to argue that is the following corollary of Carath{\'e}odory's
theorem:
\begin{obs}\label{obs:carath}
If every point from $X$ is separated from $Y$ by more than $\bigl(1-\frac{1}{d+1}\bigr)$-fraction of the $h_t$'s
then the convex hulls of $X$ and $Y$ are disjoint.
\end{obs}
\begin{proof}
First, we claim that every $d+1$ points from $X$ are separated from $Y$ by one of the $h_t$'s.
Indeed, every point in $X$ is \emph{not} separated by less than $\frac{1}{d+1}$-fraction of the $h_t$'s.
Now, a union bound yields that indeed any $d+1$ points from $X$ are separated from $Y$ by one of the $h_t$'s.

Now, this implies that $\mathsf{conv}(X)\cap\mathsf{conv}(Y)=\emptyset$:
by contraposition, if $\omega \in \mathrm{conv}(X)\cap\mathrm{conv}( Y)$ then by Carath{\'e}odory's theorem $\omega$ is in the convex hull of $d+1$ points from $X$.
 Hence these $d+1$ points can not be separated from $Y$.
\end{proof}


It remains to explain why the property 
holds when the protocol continues for $T$ iterations.
It relies on the so-called \emph{margin effect} of boosting algorithms~\cite{schapire97boosting}:
the basic result for the Adaboost algorithm states,
in the language of this problem, 
that after enough iterations, every $x\in X$ will be separated 
from $Y$ by a majority of the $h_t$'s. 
The margin effect refers to a stronger fact that 
this fraction of $h_t$'s increases when the number of iterations increases.
Our choice of $T$ guarantees that every $x\in X$ is separated
by more than a $1-\frac{1}{d+1}$ fraction of the $h_t$'s (see \cref{lem:supermajority}),
as needed.
To conclude, if $T$ iterations have passed without Bob reporting an intersection,
then the convex hulls are disjoint.

{Finally, we would also like to show how, given the output, we can calculate a separating hyperplane. Here for simplicity of the exposition we show how Alice can calculate the hyperplane, then she may transmit it by sending appropriate support vectors. 

Since each hyperplane $h_t$ was chosen so that $Y$ is contained in one of its sides,
it follows that Bob's set is contained in the intersection of all of these half-spaces.
We denote this intersection by $K_+$. 
The same argument we used above shows that $K_+$ is disjoint from Alice's convex hull
(because every point of Alice is separated from $K_+$ by more than a $1-\frac{1}{d+1}$ fraction of the $h_t$'s).
Therefore, Alice, who knows both $K_+$ and $X$, can calculate a separating hyperplane.}

The protocol that is used in the proof of {Theorem~\ref{thm:proper}}
goes along similar lines. Roughly speaking, the coVC dimension
replaces the role of Carath{\'e}odory's theorem,
and the VC dimension enables the existence of the $\eps$-nets. Because in general we cannot rely on support vectors, the general protocol we run is symmetrical, 
where both Alice and Bob transmit points to decide on a joint weak hypothesis for both samples.

\subsection{Set Disjointness}\label{sec:overviewset}

A common theme for deriving lower bounds in Yao's communication model 
and related models is via reductions to the set disjointness problem. In the set disjointness problem, we consider the boolean function $\mathrm{DISJ}_n(x,y)$, which is defined on inputs $x,y\in\{0,1\}^n$ and equals $1$ if and only if 
the sets indicated by $x,y$ are disjoint
(namely, either $x_i=0$ or $y_i=0$ for all $i$). 
A classical result in communication complexity gives a lower bound for the communication complexity of $\mathrm{DISJ}_n$.

\begin{theorem}[\cite{Kalyanasundaram92disj,Razborov92disj,Kushilevitz97book}]\label{thm:disj}
\

\begin{enumerate}
\item The deterministic and non-deterministic communication complexities of $\mathrm{DISJ}_n$ are at least $\Omega(n)$.
\item The randomized communication complexity of $\mathrm{DISJ}_n$ is $\Omega(n)$. 
\item The co-non-deterministic communication complexity of $\mathrm{DISJ}_n$
is at most $O(\log n)$.
\end{enumerate}
\end{theorem}
Though this model allows more expressive communication protocols, 
the set disjointness problem remains a powerful tool for 
deriving limitations in decision as well as search problems. 
In particular, we use it in deriving 
the separation between agnostic and realizable learning (\cref{thm:agnlb}),
and the lower bounds on the $\np$ and $\conp$ sample complexities
in terms of the VC and coVC dimensions (\cref{thm:vcnp} and \cref{thm:covconp}).

To get the flavor of how these reductions work, 
we illustrate how membership of a class in $\np$ implies that it has a finite VC dimension 
through set disjointness. 
The crucial observation is that given a shattered set $R$ of size $d$, 
a sample $S$ with points from $R$ is realizable 
if and only if it does not contain the same point with different labelings. 
We use this to show that a ``short proof'' of realizability 
of such samples imply a short NP proof for $\mathsf{DISJ}_d$.
The {argument} proceeds by identifying $x,y\in\{0,1\}^{d}$
with samples $S_X,S_Y$ negatively and positively labelled respectively. With this identification, $x,y$ are disjoint if and only
the joint sample $\concatx{X}{Y}$ is realizable.
Now, since all the examples are from $R$,
a proof with $k$ examples that $\concatx{X}{Y}$
is realizable can be encoded using $k\log d$ bits
and can serve as a proof that $x,y$ are disjoint
in Yao's model.
\cref{thm:disj} now implies that the non-determinstic sample complexity
is $k \geq \Omega(d/\log d)$.

\subsection{Convex Set Disjointness}\label{sec:cvxoverview}

Here we outline our construction that is used in \cref{thm:reallb} and \cref{thm:logn}. 
The underlying hardness stems from the \emph{convex set disjointness} problem, 
where each of Alice and Bob gets a subset of $n$ points in the plane
and they need to determine whether the two convex hulls are disjoint. 
In what follows we state our main result for the convex set disjointness problem, and briefly overview the proof. 
However,  we first discuss how it is used to derive the lower bound in \cref{thm:reallb} for learning in the realizable setting.

\paragraph{From Decision Problems to Search problems.}
A natural approach to derive lower bounds for search problems is via lower bounds for corresponding
decision problems.
For example, in order to show that no proper learning protocol of sample complexity $T(1/\eps)$
for a class $\H$ exists, it suffices to show that the realizability problem for $\H$
can not be decided with sample complexity $O(T(n))$.
Indeed, one can decide the realizability problem by plugging $\eps < 1/n$ 
in the proper learning protocol,  simulating it on an input sample $S=\concat$,
and observing that the output hypothesis $h$ satisfies $L_S(h)=0$ if and only if
$S$ is realizable. Checking whether $L_S(h)=0$ can be done with just two bits
of communication. 

The picture is more complicated if we want to prove lower bounds
against improper protocols, which may output $h\notin \H$ (like in \cref{thm:reallb}). 
To achieve this, we consider a promise-variant of the realizability problem.
Specifically, we show that it is hard to decide realizability, even under
the promise that the input sample is either (i) realizable 
or (ii) contains a point with two opposite labeling.
The crucial observation is that {any} (possibly improper) learner with $\epsilon < \frac{1}{n}$ 
can be used to distinguish between case (i), for which the learner outputs $h$ with $L_S(h)=0$,
and case (ii), for which \emph{any} $h$ has $L_s(h)\geq 1/n$, where $n$ is the input sample size.

%
%

\paragraph{Main Lemma and Proof Outline.} 
The above promise problem is stated as follows in the language of convex set disjointness.
\begin{restatable}[Convex set disjointness lower bound]{lemma}{cvxdisjoint}\label{lem:cvxdisjoint}
Consider the convex set disjointness problem in $\R^2$, where Alice's input is denoted by $A$,
Bob's input is denoted by $B$, and both $\lvert A\rvert,\lvert B\rvert$ are at most $n$.
Then any communication protocol with the following properties must have sample complexity at least~$\tilde\Omega(\log n)$.
\begin{enumerate}[(i)]
\item Whenever $\mathsf{conv}(A)\cap\mathsf{conv}(B) = \emptyset$ it outputs 1.
\item Whenever $A\cap B\neq \emptyset$ it outputs 0.
\item It may output anything in the remaining cases. 
\end{enumerate}
\end{restatable}

\ignore{
\begin{lemma}[Convex Set Disjointness -- Lower Bound]\label{lem:cvxdisjoint}
There exists a distribution over pairs of random sets $(A,B)$ of size $n$ in $\R^2$, such that 
\begin{enumerate}
\item\label{cvx:1} $A\cap B\ne \emptyset$ or the convex hull of $A$ and $B$ are disjoint.
\item\label{cvx:2} Assume Alice receives as sample $S_a= \{(a,1): a\in A\}$ and Bob receives as sample $S_{b}=\{(b,-1), b\in B\}$. Then any communication protocol that determines correctly the realizability problem w.p. $>0.9$, must send $\Omega(\log(n)/\log \log(n))$ points.
\end{enumerate}
\end{lemma}
}
We next try to sketch the high level idea of the proof
(so we try to focus on the main ideas
rather than on delicate calculations).
The complete proof is somewhat involved and appears in \cref{sec:cvxdisjoint}. 

Like in our other lower bounds, we reduce the proof to a corresponding problem
in Yao's model. 
A challenge that guides the proof is that the lower bound should apply
against protocols that may send examples, which contain a large number of bits (in Yao's model). 
Note that in contrast with previous lower bounds, we aim at showing an $\Omega(\log n)$ bound, which roughly corresponds to the bit capacity of each example in a set of size $n$. Thus, a trivial lower bound showing $\log n$ bits are necessary may not suffice to bound the sample complexity.
This is handled by deriving a round-communication tradeoff,
which says that every $r$-rounds protocol for this problem has complexity of at least 
$\tilde\Omega(r + n^{1/r})$.
This means that any efficient protocol must have many rounds,
and thus yields \cref{lem:cvxdisjoint}.

The derivation of this tradeoff involves embedding 
a variant of ``pointer chasing'' in the Euclidean plane
(see~\cite{papa84comm,nisan93rounds} for the original variant of pointer chasing). 
The hard input-instances are built via a recursive construction
(that allows to encode tree-like structures in the plane). 

For integers $m,r>0$ we produce a distribution over inputs $I_{m,r}=(A_{m,r},B_{m,r})$ 
of size $n\approx m^r$.
We then show that a random input from $I_{m,r}$ cannot be solved 
in fewer than $r$ rounds,
each with sample complexity less than $m$
(the exact bounds are not quite as good).

For $m=r=1$, we set $A_{m,r}=\{(0,0)\}$ and $B_{m,r}$ randomly either $\{(0,0)\}$ or $\emptyset$. 
If there is only one round in which Alice speaks, then she cannot determine whether or not their sets intersect
and therefore must err with probability $1/2$. 
For $r>1$, we take $m$ points spaced around a semicircle (an example with $m=3$ is shown in \cref{fig:cvxdisjoint}). 
Around each point we have a dilated, transposed and player swapped copy of $I_{m,r-1}$. 
Alice's set is the union of the copies of points of the form $B_{m,r-1}$ in all of the 
$m$ copies, 
while Bob's set are points of the form $A_{m,r-1}$ in a single copy $i$ chosen at random.
We rotate and squish the points so that any separator of the
Bob's points from the $i$'th copy of $B_{m,r-1}$ that Alice holds,
will also separate Bob's points from the rest of Alice's points. 
This guarantees that 
all of Alice's copies except the $i$'th one will not affect whether or not the convex hulls intersect,
which means that to solve $I_{m,r}$ they must solve a single random copy
of $I_{m,r-1}$.

\begin{figure}
\begin{center}
\includegraphics[scale = 0.5]{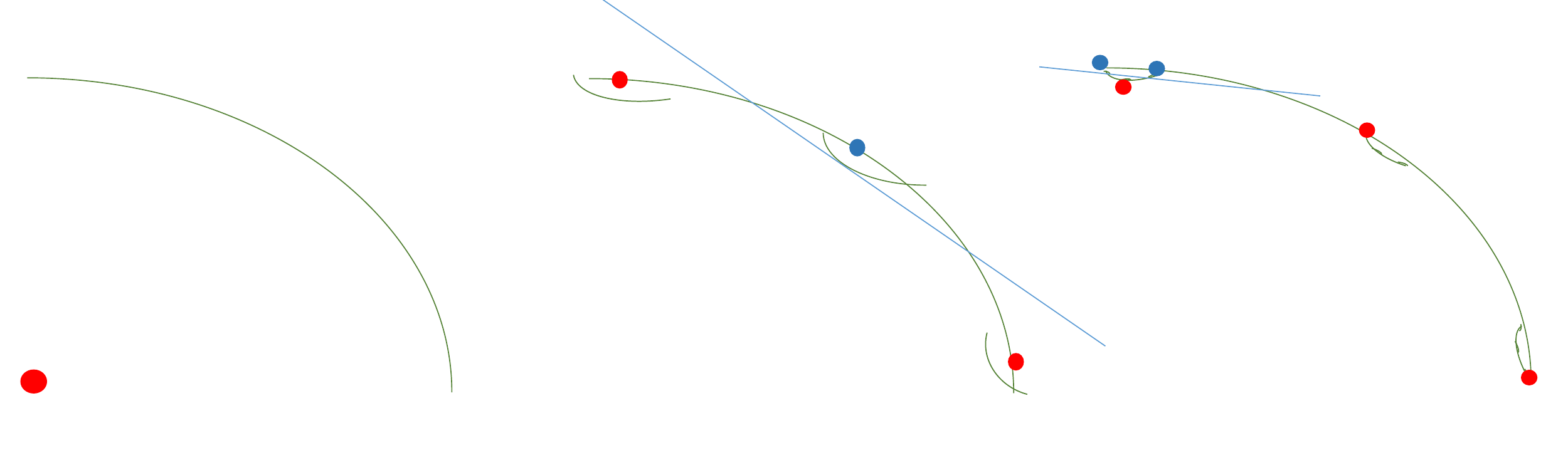}
\end{center}
 \captionsetup{width=.8\linewidth}
\caption{\small{Examples of separable instances $I_{3,1},I_{3,2}, I_{3,3}$. In each instance, Alice's points are red and Bob's points are blue. In $I_{3,1}$ Alice has a single point, and Bob an empty set. 
In $I_{3,2}$, Alice receives three instances of
a Bob's points in $I_{3,1}$ (the first and last instances are a point and the middle instance is an empty set), and Bob receives a single instance (a single point). These are then embedded around $3$ points on the sphere. This construction continues to $I_{3,3}$ where again roles are reversed. To maintain separability, the instances are rotated in the plane.}}\label{fig:cvxdisjoint}
\end{figure}

The proof now proceeds via a round elimination argument. 
We can think of Alice's set as consisting of $m$ instances of 
a small problem (with parameter $r-1$)
and Bob's set as consisting of a single instance of Alice's,
chosen uniformly and independently of other choices
(represented by $i$). 
Alice's and Bob's convex hulls overlap if and only if the convex hulls of 
the copies that correspond to the single instance that Bob holds overlap.
Thus, assuming Alice speaks first, since she does not know $i$, 
her message will provide negligible information on the $i$'th copy,
unless her message is long. 
This is formalized using information theory in a rather standard way.


%% file: discussion.tex
\section{Discussion and Future Research}
\label{sec:discussion}

\subsection{$\p = \np\cap\conp$ for Realizability Problems}\label{sec:npconp}
Theorem~\ref{thm:pchar} states that 
$\p=\np\cap\conp$ in the context of realizability problems. 
An analogous result is known to hold in standard communication complexity as well~\cite{Aho83notions};
this result is more general than ours in the sense that it applies to arbitrary decision problems, 
while \cref{thm:pchar} only concerns realizability problems. 

It is natural to ask how these two results are related,
and whether there is some underlying principle
that explains them both.
While we do not have a full answer,
we wish to highlight some differences between the two theorems.
%

First, the proofs are quite different.
The proof by \cite{Aho83notions}~is purely combinatorial 
and relies on analyzing coverings of the input space
by monochromatic rectangles.
Our proof of \cref{thm:pchar} uses fractional combinatorics;
in particular it is based on linear programming duality
and multiplicative weights update regret bounds.

Second, 
the theorem from \cite{Aho83notions} gives a protocol
with bit-complexity $O(N_0\cdot N_1)$,
where $N_0,N_1$ are the non--deterministic complexities.
\cref{thm:pchar} however gives a protocol
with sample complexity $\tilde O(S_0^2 S_1\log n)$,
where $S_0,S_1$ are the non--deterministic sample complexities,
and $n$ is the input size.
The former bound 
is also symmetric in $N_0,N_1$ while the latter bound
is not symmetric in $S_0,S_1$.
{This difference may be related to
that} while the negation of a decision problem
is a decision problem,
there is a no clear symmetry
between a realizability problem 
and its negation 
(i.e.~the negation is not a realizability problem with respect to another class).



\subsection{Sample Compression Schemes as One-Sided Protocols}\label{sec:samplecompression}

Sample compression schemes were introduced by~\cite{Littlestone86relating} 
as a convenient framework for proving generalization bounds for classification problems,
and were studied by many works (a partial list includes~\cite{FW95, DL98,LS13, gottlieb14near,balcan14learning,MSWY15,wiener15agnostic,bendavid16version,kontorovich17nearest,ashtiani17agnostic})
%
In the {two-party model considered here}, they correspond to one-sided protocols where only one party
receives a realizable input sample $S$ (the other party's input is empty),
and the goal is to transmit as few example as possible so that the
receiving party can output a hypothesis that is consistent with all the input sample.
The {\em size} of a sample compression scheme is the number of transmitted examples.

Thus, a possible way to view this model 
is as \emph{distributed sample compression schemes}
(i.e.\ the input sample is distributed between the two parties).
With this point of view, Theorem~\ref{thm:reallb}
implies that every distributed sample compression scheme
for half-planes must have size $\tilde\Omega(\log n)$;
in particular, the size must depend on the input sample size.
This exhibits a difference with (standard, one-sided) sample compression schemes
for which it is known that every class has 
{a compression scheme} of size 
depending only on the VC dimension of the class~\cite{MY15};
{specifically, half-planes have compression schemes of size $3$ 
(using the support vectors), but Theorem~\ref{thm:reallb}
implies that every distributed sample compression scheme
for them must have size $\tilde \Omega(\log n)$.}
%

{\subsection{Decision Vs. Search}
One important aspect of this work is 
the connection between search and decision problems. 
For example,
\cref{thm:proper} provides, in the realizable setting, an equivalence between
efficient proper learnability and the analogous decision problem of realizability.

A similar equivalence holds in the improper realizable setting,
where the associated decision problem is 
the following \emph{promise} variant of the realizability
problem: distinguish between 
(i) realizable samples,
and (ii) noisy samples 
(where a sample is noisy if it contains two examples with the same point and opposite labels). 
Indeed, any efficient (possibly improper) learner implies
an efficient protocol for this promise variant 
(this is used in the proof of \cref{thm:reallb}).
Also the opposite direction holds: 
an efficient protocol for the promise problem implies a finite VC dimension, 
which yields an efficient learning protocol via boosting. 

It is interesting to compare these connections to recent works 
that relate efficient PAC learnability to refutation problems. 
Concretely, \cite{vadhan} shows that a class
is efficiently (possibly improperly) PAC-learnable in the realizable setting 
if and only if there exists an efficient algorithm that distinguishes between 
(i) realizable samples, 
and (ii) noisy samples where here noisy means that the labels are independent of points.
A similar result by \cite{kothari} relates agnostic learning and distinguishability between (i) samples that correlates with the class and 
(ii) noisy samples (i.e.\ the labels are independent from the points). 
}
\subsection{Open Questions}\label{sec:openquestions}

\paragraph{The complexity of convex set disjointness.}
There is a gap between our upper bound 
and lower bound on the sample complexity of the convex set disjointness problem;
$\tilde O(d^3\log n)$ versus $\tilde \Omega(d +\log n)$ by  \cref{thm:pchar} and \cref{thm:reallb}. More generally, it would be interesting to obtain tight bounds for proper learning of classes with finite VC and coVC dimensions.

\paragraph{Combinatorial characterizations of $\np$.}
Theorems~\ref{thm:pchar} and~\ref{thm:covconp} give a combinatorial characterization of $\p$ and $\conp$.
Indeed, Theorem~\ref{thm:pchar} shows that $\H$ is in $\p$
if and only if it has finite VC and coVC dimensions,
and Theorem~\ref{thm:covconp} shows that $\H$ is in $\conp$ if and only if it
has a finite coVC dimension.

It would be interesting to find such a characterization
for the class $\np$ as well.
Theorem~\ref{thm:vcnp} implies that
every class in $\np$ has a finite VC dimension |
the converse remains open.

A related open problem is the existence of \emph{proper sample compression schemes}. 
Indeed, the existence of proper compression scheme of polylogarithmic sample size
will entail that every VC class is  in $\np$.


%

%

\paragraph{Agnostic learning.}
Theorem~\ref{thm:agnub} shows that every VC class can be
learned in the agnostic case with sample complexity $o(1/\eps^2)$.
However, this is still far from the lower bound given in
Theorem~\ref{thm:agnlb} of $\tilde \Omega(1/\eps)$.
It would be interesting to find the correct dependency.

The protocol from Theorem~\ref{thm:agnub}
reveals much more information than required. 
Indeed, the subsample published
by the parties forms an $\eps$-approximation, and therefore
reveals, up to $\pm\eps$ the losses of all hypotheses in $\H$,
rather than just the minimizer.
Also, the protocol uses just one round of communication.
Therefore, it is plausible that the bound
in Theorem~\ref{thm:agnub} can be improved.

Another interesting direction concerns relaxing the definition
of agnostic learning by allowing a multiplicative slack.   
Let $c\geq 1$ be a constant. We say that a protocol
$c$-agnostically learns $\H$ if for every input sample $S=\concat$
it outputs $h$ such that $L_S(h)\leq c \cdot \min_{f \in H}L_S(f) + \eps$.
What is the sample complexity of $c$-agnostically learning a class of VC dimension $d$?
As mentioned above, for $c > 4$ 
the sample complexity is $O(d\log(1/\eps))$ by \cite{Chen16boosting},
and for $c=1$ it is $\tilde\Omega(1/\eps)$ by \cref{thm:agnlb}.

\paragraph{Learning (noiseless) concepts.}
Our lower bounds rely on hard input samples
that are noisy in the sense that they contain a point 
with opposite labels. It would be interesting
to study the case where the input sample is guaranteed
to be consistent with some hypothesis~$h$ (not necessarily in $H$).
As a simple example let $H_d$ be the class of all concepts with at most $d$
many $1$'s. Since the VC dimension of $H$ is $d$,
it follows that deciding the realizability problem
for $H$ has $\tilde \Omega(d)$ sample complexity.
However, if the input sample is promised to be noiseless
then there is an $O(\log d)$ protocol for deciding
realizability problem. Indeed, the parties just need to check whether
the total number of $1$'s in both input samples is at most~$d$ or not.

Similarly, our lower bound in the agnostic case uses noisy samples,
and it could be that agnostic learning is easier for noiseless input.
Let $\H$ be a class with a finite VC dimension.
Consider the problem of agnostically learning under the promise
that the input sample is consistent with some target function.
Is there a learning protocol in this case with sample complexity $o(1/\eps)$?

\paragraph{Multiclass categorization.}
The model presented here naturally extends to multiclass categorization,
which concerns hypotheses $h:X\to Y$ for large $Y$.
Some of the arguments in this paper {naturally} generalize,
while others less so.
For example it is no longer clear whether $\p=\np\cap\conp$
when the range $Y$ is very large (say $Y=\N$).

\section*{Acknowledgements}
We thank Abbas Mehrabian and Ruth Urner for insightful discussions.

%% file: proofs-decision.tex

\section{Technical Background}
\subsection{Boosting and Multiplicative Weights}
Our communication protocols in the realizable setting are 
based on the seminal \emph{Adaboost} algorithm~\cite{freund97decision}, which we briefly outline next. 
The simplified version of Adaboost we apply here may be found in \cite{schapire2012boosting}.
%

Adaboost gets as an input a sample $S$ 
and outputs a classifier.
It has an oracle access to an \emph{$\alpha$-weak} learner.
This oracle gets as input a distribution $p$
over $S$ and returns a hypothesis $h=h(p)$
that has an advantage of at least $\alpha$ over a random guess, namely:
\begin{align*}
\mathbb{E}_{(x,y)\sim p} \left[\err{h}\right] \le \frac{1}{2}- \alpha.
\end{align*}
Adaboost proceeds in rounds $t=1,2,\ldots,T$. 
In each round it calls the weak learner with a distribution $p_t$ 
and receives back a hypothesis $h_t$.
Its output hypothesis
is the point-wise majority vote of all the $h_t$'s.
To complete the description of Adaboost, 
it remains to describe how the $p_t$'s are defined: 
$p_1$ is the uniform distribution over $S$, and for every $t>1$ and $z=(x,y)\in S$ we define by induction:
\[p_{t+1}(z)\propto p_t(z)e^{-\eta \cdot 1[h_t(x)=y]}\]
where $\eta$ is a parameter of choice.

Thus, $p_{t+1}$ is derived from $p_t$ by decreasing the probabilities
of examples on which $h_t$ is correct and increasing the probabilities of examples
where $h_t$ is incorrect.

The standard regret bound analysis of boosting yields
\begin{theorem}[\cite{freund97decision}]\label{thm:boosting}
Set the parameter $\eta$ in Adaboost to be  $\alpha$. 
Let $\eps > 0$ and 
let $T\geq \frac{2\ln(1/\eps)}{\alpha^2}$. 
Let $h_1,\ldots, h_T$ denote the weak hypotheses returned by 
an arbitrary $\alpha$--weak learner during the execution of Adaboost. 
Then, there is $S'\subseteq S$ of size $\lvert S' \rvert \geq (1-\eps)\lvert S\rvert$
such that  for every~$(x,y)\in S'$:
\begin{align*}
 \frac{1}{T}\sum_{t=1}^T \err{h_t} < 1/2.
\end{align*}
\end{theorem}
In other words, after $O\bigl({\log(1/\eps)}/{\alpha^2}\bigr)$
rounds $L_S(h) \leq \eps$, 
where $h$ denotes the majority vote of the $h_t$'s.
In particular, $L_S(h)=0$ after $T=O\bigl(\log(|S|)/\alpha^2\bigr)$
rounds.

By a simple extension 
to the standard boosting analysis, it is well known that adding sufficiently many rounds, even after the error rate is zero, leads to a super-majority of the hypotheses to become correct on \emph{every} point in the sample. Using a more refined analysis, one can improve the convergence rate for sufficiently strong learners and, for completeness, we perform this in the next lemma (see also \cite{schapire2012boosting} and the analysis of $\alpha$-boost for similar bounds).

\begin{lemma}\label{lem:supermajority}
Set the parameter $\eta$ in Adaboost to be  $\ln 2$. 
Let $T\geq 2 k \log |S|$ for $k>0$, 
and have $h_1,\ldots, h_T$ denote the weak hypotheses returned 
by an arbitrary $\alpha$-weak learner with $\alpha = \frac{1}{2}-\frac{1}{5k}$
during the execution of Adaboost. 
Then,  for every~$z=(x,y)\in S$:
%
\begin{align*}&\label{eq:supermajority}
 \frac{1}{T}\sum_{t=1}^T \err{h_t} \le \frac{1}{k}.
\end{align*}
\end{lemma}
\begin{proof}
For each $z\in S$ set $W_1(z)=1$ and, for each $t>1$, set
$$W_{t+1}(z) =W_t(z) 2^{-1[h_t(x)=y]}.$$
By choice of $\eta$ and the update rule we have that $p_t(z) = \frac{W_t(z)}{\Phi_t}$,
where $\Phi_t= \sum_{z\in S} W_t(z)$. 
Next, since $h_t$ is $\alpha$--weak with respect to $p_t$ we have that $\sum_{h_t(x)\ne y} p_t(z)\le \frac{1}{5k}$. 
Thus:
%
\begin{align*}
\Phi_{t+1} 
&= \sum_{\{z\in S: h_t(x)= y\}} W_{t+1}(z) +\sum_{\{z\in S: h_t(x)\ne y\}} W_{t+1}(z)&\\
&=\sum_{\{z\in S: h_t(x)= y\}} \frac{1}{2}W_{t}(z)+\sum_{\{z\in S: h_t(x)\ne y\}} W_{t}(z)\\
&= \Phi_t \cdot\left(\sum_{\{h_t(x)= y\}} \frac{1}{2}\frac{W_{t}(z)}{{\sum_{z\in S}{W_t(z)}}}+\sum_{\{h_t(x)\ne y\}} \frac{W_{t}(z)}{\sum_{z\in S}{W_t(z)}}\right)& \left( \Phi_t = \sum_{z\in S}W_t(z)\right)&\\
& = \Phi_t \cdot\left(  \sum_{\{z\in S: h_t(x)= y\}}\frac{1}{2}p_t(z)+\sum_{\{z\in S: h_t(x)\ne y\}}p_t(z)\right) & \left( p_t(z) = \frac{W_t(z)}{\sum_{z\in S} W_t(z)} \right) \\ 
& = \Phi_t \cdot\left(\sum_{z\in S}\frac{1}{2}p_t(z)+\sum_{\{z\in S: h_t(x)\ne y\}}\frac{1}{2}p_t(z) \right)\\
&\le \frac{\Phi_t}{2}\cdot \left(1+\frac{1}{5k}\right) & \left( \sum_{\{h_t(x)\ne y\}} p_t(z)\le 1/(5k) \right) &\\ 
& \le \Phi_t \cdot 2^{-1+1/(2k)} . & \left( 1+1/(5k)\le 2^{1/(2k)} \right)
\end{align*}
By recursion, we then obtain $\Phi_T \le |S| 2^{-T(1-1/(2k))}.$
Thus, for  every $z = (x,y)\in S$,
$$|S| 2^{-T(1-1/(2k))} \geq \Phi_T > W_T(z)=2^{-\sum_{t} 1[h_t(x)=y]}
.$$ 
Taking log and dividing by $T$ we obtain
\begin{align*}
\frac{\log |S|}{T}  +\frac{1}{2k} -1 &> -\frac{1}{T}\sum_{t=1}^T1[h_t(x)=y] \\& =  -(1 -\frac{1}{T}\sum_{t=1}^T 1[h_t(x)\ne y])
\end{align*}
Rearranging the above and setting $T=2k\log |S|$ we obtain the desired result.
\end{proof}

\subsection{$\eps$-nets and $\eps$-approximations}

{We use standard results from VC theory and discrepancy theory.}
Throughout this section if $p$ is a distribution over a sample $S$,  then we let $L_p(h):= \mathop{\mathbb{E}}_{z\sim p} 1[h(x)\ne y]$ denote the expected error of a hypothesis $h$ w.r.t distribution $p$. 
{In the realizable setting we use the following $\eps$-net Theorem.}

\begin{theorem}[\cite{Haussler87epsilon}]\label{thm:epsnet}
Let $\H$ be a class of VC dimension $d$ and let $S$ be a realizable sample.
For every distribution $p$ over $S$ there exists a subsample $S'$ of $S$
of size $O\bigl(\frac{d\log(1/\eps)}{\eps}\bigr)$ such that
\[\forall h\in H :  L_{S'}(h) = 0 \implies  L_p(h) \leq \eps .\]
\end{theorem}

In the agnostic setting we use the stronger notion of \emph{$\eps$-approximation}.
The seminal uniform convergence bound due to Vapnik and Chervonenkis~\cite{Vapnik69uniform}
states that for every class $\H$ with $\vc(\H)=d$,
and for every distribution $p$ over examples,
a typical sample~$S$ of~$O(d/\eps^2)$ indepenent examples from~$p$ 
satisfies that~$\forall h\in H :\lvert L_S(h)-L_p(h)\rvert\leq\eps$.
This result is tight when $S$ is random,
however, it can be improved
if $S$ is constructed systematically:

\begin{theorem}[\cite{Matousek93discrepancy}]\label{thm:epsapprox}
Let $\H$ be a class of VC dimension $d$ and let $S$ be a sample.
For every distribution $p$ over $S$ there exists a subsample $S'$ 
of size $O_d \bigl((1/\eps)^{2-\frac{2}{d+1}}(\log(1/\eps))^{2-\frac{1}{d+1}}\bigr)$ such that
\[\forall h\in H :\lvert L_S(h)-L_p(h)\rvert\leq\eps,\]
where $O_d(\cdot)$ hides a constant that depends on $d$.
\end{theorem}

\section{Search Problems: Proofs}

\stoptocentries
\subsection{Proof of \cref{thm:reallb}}\label{prf:reallb}
\reallb*

\begin{proof}
We begin by showing that $\tilde\Omega(d)$ examples are required, even for $\eps=1/3$.
The argument relies on the relation between VC dimension and compression schemes.
In the language of this paper, a compression scheme is a one-sided protocol 
in the sense that only Alice gets the input sample (i.e.\ $S_a=S$, $S_b=\emptyset$). 
An \emph{$\eps$-approximate sample compression scheme}
is a sample compression scheme with $L_S(h)\leq\eps$ where $h$ is the output hypothesis. 
A basic fact about $\eps$-sample compression schemes is that for any fixed $\eps$, say $\eps=1/3$,
their sample complexity is $\tilde\Omega(d)$, where $d$ is the VC dimension (see, for example,~\cite{david16statistical}).
Now, assume $\Pi$ is a protocol with sample complexity $C$ and error $\leq 1/3$.
In particular, $\Pi$ induces an $\eps=1/3$-compression scheme and so $C=\tilde\Omega(d)$.
%
%

We next set out to prove that $\tilde\Omega(\log 1/\epsilon)$ samples are necessary. 
The proof follows directly from \cref{thm:lognstrong}. Indeed setting $\epsilon=\frac{1}{n}$,  let $\Pi$ be a protocol that learns $\H$ to error $\epsilon$, using $\tilde O(T(n))$ samples. Then we construct a protocol $\Pi'$ whose sample complexity is $\tilde O(T(n))$ that satisfies the premises in \cref{thm:lognstrong} as follows: $\Pi'$ simulate $\Pi$ over the sample and considers if the output $h^*$ satisfies $L_{S}(h^*)>0$, which can be verified by transmitting {two additional} bits. 
{The protocol $\Pi'$ indeed satisfies the premises in \cref{thm:lognstrong}} | if the sample $S$ contains two points $(x,-1),(x,1)\in S$ then clearly $\L_{S}(h^*)>0$, and otherwise if the sample is realizable then $L_{S}(h^*)=0$ by choice of~$\epsilon$.


\end{proof}

%
\subsection{Proof of \cref{thm:properlb}}\label{prf:properlb}
\properlb*
\begin{proof}
To prove  \cref{thm:properlb} we use the following construction of a class with infinite coVC dimension and VC dimension $1$:
set $\X= \{(m,n): m\le n, m,n\in \mathbb{N}\}$, and define a hypothesis class {$\H =\{h_{a,b} : a\leq b\}$},
where
\begin{align*}
h_{a,b}((m,n)) = 
\begin{cases}
1 & n\ne b \\
1 & n= b ,  m=a \\
-1 & \mathrm{else}
\end{cases} 
\end{align*}
Roughly speaking, the class $\H$ consists of infinitely many copies of singletons
over a finite universe.
The VC dimension of $\H$ is $1$.
To see that the coVC dimension is unbounded, take 
\[S_k=\Bigl( \bigl((1,k),-1\bigr), \bigl((2,k),-1\bigr),\ldots, \bigl((k,k),-1\bigr)\Bigr).\] 
The sample $S_k$ is not realizable. However, every subsample of size $k-1$ does not include some point 
of the form $(j,k)$, so it realizable by $h_{j,k}$.

Therefore, by \cref{thm:covconp} it follows that the $\conp$ sample complexity of this class is $\tilde\Omega(n)$ for inputs of size $n$.
Thus, deciding the realizability problem for this class requires sample complexity
$\tilde\Omega(n)$. This concludes the proof, because any protocol that properly learn this class 
yields a protocol for the realizability problem by simulating the proper learning protocol with $\eps = 1/(2n)$
and testing whether its output is consistent.
\end{proof}
%
%

\subsection{Proof of \cref{thm:agnlb}}\label{prf:agnlb}
\agnlb*
The VC dimension~1 class is the class of singletons over $\N$;
it is defined as $\H= \{h_{n} :n\in\N\}$ where
\[h_n(x)=
\begin{cases}
1 &x=n \\
-1  &x\neq n.
\end{cases}\]
The proof relies on the following 
reduction to the set disjointness problem in Yao's model; we defer its proof to the end of this section. 
\begin{lemma}\label{lem:agntodisj} 
{There are two maps $F_a,F_b:\{0,1\}^n \to \left([n]\times \{\pm 1\}\right)^n $, 
from $n$ bit-strings to samples of size $n$,
for which the following holds:}
Let $x,y \in \{0,1\}^n$, and set $S=\concatxx{F_a(x)}{F_b(y)}$. Then
\begin{enumerate}
\item\label{agnostic:1} If $x\cap y = \emptyset$ then $L_{S}(f)\ge\frac{|x|+|y|}{2n}$
 for every $f:[n]\to \{\pm 1\}$.
\item\label{agnostic:2} If $x\cap y \ne \emptyset$ then 
$L_{S}(f) \leq \frac{|x|+|y|-2}{2n}$ for some $h\in \H_n$.
\end{enumerate}
\end{lemma}
With this lemma in hand, we prove \cref{thm:agnlb}.

\begin{proof}[Proof of \cref{thm:agnlb}]
Assume that the class of singletons on $\N$ {can be learned in the agnostic setting}
by a protocol with error $\epsilon$ and sample complexity $T(1/\epsilon)$.
We derive a protocol for deciding $\mathrm{DISJ}_n$ in Yao's model using $O(T(n)\log n)$ bits.

Let $\Pi$ be a protocol that learns $\H$ up to error $\epsilon =1/(4n)$
by sending $T(4n)$ bits.
{The first observation is that by restricting the input sample to contain only examples from $[n]\times\{\pm 1\}$,
we can simulate $\Pi$  by a protocol in Yao's model that sends $O(T(4n)\log n)$ bits.
Next, define a protocol $\Pi'$ for $\mathrm{DISJ}_n$ as follows.}
\begin{itemize}
\item Alice is given $x \in \{0,1\}^n$ and Bob is given $y \in \{0,1\}^n$.
\item The player transmit the sizes $|x|$ and $|y|$ using $O(\log n)$ bits.
\item The two parties simulate the learning protocol $\Pi$ with $S_a = F_a(x)$ and $S_b=F_b(y)$.
\item Let $h$ denote the output of $\Pi$.
The players transmit the number of mistakes of $h$ over the sample $F_{a}(x)$ and $F_{b}(x)$ using $O(\log n)$ bits.
\item Alice and Bob output $\mathrm{DISJOINT}$ if and only if $L_{S}(h) \geq \frac{|x|+|y|-1}{2n}$.
\end{itemize}
Since $L_{S}(h) \le \min_{f\in \H}L_{S}(f)+ \frac{1}{4n}$, by \cref{lem:agntodisj},
the protocol $\Pi'$ outputs $\mathrm{DISJOINT}$ if and only if 
$$\min_{f\in \H}L_{S}(f) > \frac{|x|+|y|-2}{2n}.$$  
In addition, the optimal hypothesis in $\H$ has error at most $\frac{|x|+|y|-2}{2n}$ if and only if the two sets are not disjoint.
So $\Pi'$ indeed solves $\mathrm{DISJ}_n$.
\cref{thm:disj} now implies that $T(1/\epsilon)=\tilde\Omega(\frac{1}{\epsilon})$.
\end{proof}

\begin{proof}[Proof of \cref{lem:agntodisj}]
Given two bit-strings $x,y \in \{0,1\}^n$, let $F_{a}(x)$ be the sample 
$$((1,a_1),(2,a_2),\ldots, (n,a_n))$$ 
where $a_i=(-1)^{1-x_i}$ for all $i$. Similarly 
define $F_{b}(y)$.
Let $h:\N\to\{\pm 1\}$ be any hypothesis. 
For any $i\in x\Delta y$, where $\Delta$ is the symmetric difference,
we have that $a_i\neq b_i$ and so $h$ is inconsistent either with $(i,a_i)$ or with $(i,b_i)$.
Thus, if $x\cap y=\emptyset$, then $L_{S}(h)\geq \frac{|x|+|y|}{2n}$ for any hypothesis $h$. 
On the other hand, if $x\cap y$ is non empty, then any singleton $h_{i}$ for $i\in x\cap y$
has error $L_{S}(h_i) = \frac{|x\cup y|-2}{2n} \leq \frac{|x|+|y|-2}{2n}$.
\end{proof}

\subsection{Proof of \cref{thm:agnub}}\label{prf:agnub}
\agnub*

\begin{proof}
\cref{thm:epsapprox} implies a one round proper agnostic learning protocol
that we describe in~\cref{fig:algagn}.
To see that $h$, the output of this protocol satisfies $L_S(h) \leq \min_{f\in H}L_S(f)+\eps$,
use Theorem~\ref{thm:epsapprox} and the fact that
$L_S=\frac{\lvert S_a\rvert}{\lvert S\rvert}L_{S'_a}(h) + \frac{\lvert S_b\rvert}{\lvert S\rvert}L_{S'_b}(h)$. 
\end{proof}

\begin{figure}
\begin{tcolorbox}
\begin{center}
{\bf An $o_d(1/\eps^2)$ agnostic learning protocol}\\
\end{center}
\textbf{Input}: A joint input sample $S=(S_a,S_b)$ that is realizable by $\H$, and $\eps > 0$. \\ \ \\
\textbf{Protocol:}
\begin{itemize}
\item Alice and Bob transmit the sizes $\lvert S_a\rvert$
and $\lvert S_b\rvert$.
\item Each of Alice Alice and Bob finds subsamples $S'_a,S'_b$ 
like in Theorem~\ref{thm:epsapprox} with parameter $\eps$ and transmit it.
\item Alice and Bob agree (according to a predetermined ERM rule)
on $h\in H$ that minimizes
$\frac{\lvert S_a\rvert}{\lvert S\rvert}L_{S'_a}(h) + \frac{\lvert S_b\rvert}{\lvert S\rvert}L_{S'_b}(h)$, 
and output it.
\end{itemize}
\end{tcolorbox}
\caption{A learning protocol in the agnostic case}\label{fig:algagn}
\end{figure}

\starttocentries
\section{Decision Problems: Proofs}
\stoptocentries
\subsection{Proof of \cref{thm:vcnp}}\label{prf:vcnp}
\vcnp*

We use the following simple lemma.

\begin{lemma}\label{lem:npreduction}
Let $\H$ be a hypothesis class and let $R\subseteq X$
be a subset of size $n$ that is shattered by $\H$. 
There exists two functions $F_a,F_b$ 
that map $n$ bit-strings to labelled examples from $R$
such that for every $x,y\in\{0,1\}^n$, it holds that
$x\cap y = \emptyset$ if and only if
the joint sample $S=\concatxx{F_a(x)}{F_b(y)}$ is realizable by $\H|_R$.
\end{lemma}
\begin{proof}[Proof of \cref{lem:npreduction}]
Since $R$ is shattered by $\H$, it follows that a sample $S$ with examples from $R$
is realizable by $\H$ if and only if it contains no point with two opposite labels.
Now, identify $[n]$ with $R$.
Set $F_a(x) = \{(i,1):  ~x_i=1\}$ and set
$F_b$ in the opposite manner: namely, $F_b(y)=\{(i,-1):~ y_i=1\}$.


If $i \in x \cap y$ then having {$(i,1)\in F_a(x)$ and $(i,-1)\in F_b(y)$} implies that the joint sample
$S$ is not realizable. 
On the other hand, since $R$ is shattered, we have that if $x\cap y=\emptyset$, then $S$ is realizable.
\end{proof}

\begin{proof}[Proof of \cref{thm:vcnp}]
Let $R$ be a shattered set of size $d$.
Since every example  $x\in R$ can be encoded
by $O(\log d)$ bits, it follows that every $\np$-proof
of sample complexity $T$ for the realizability problem
for $\H|_R$ implies an $\np$-proof 
for $\mathsf{DISJ}_d$ with bit-complexity $O(T\log(d))$ in Yao's model.
{This concludes the proof} since the non-deterministic communication complexity of 
$\mathsf{DISJ}_d$ is $\Omega(d)$, by~\cref{thm:disj}.
\end{proof}

%

\subsection{Proof of Theorem~\ref{thm:covconp}}\label{prf:covconp}
\covcconp*

\begin{proof}
For the direction $\Tc(n) \leq k$, assume that the coVC dimension is $k<\infty$.
If $S = \concat$ is not realizable 
then it contains a non realizable sample $S$ of size at most $\covc(\H)=k$
that serves as a proof that $S$ is not realizable. 
{If $k=\infty$ then the whole sample $S$ serves as a proof of size $n$
that it is not realizable.}

The other direction follows by a reduction from the set disjointness problem in Yao's model
(similarly to the proof of \cref{thm:vcnp} from the previous section).
We use the following lemma.

\begin{lemma}\label{lem:conpreduction}
Let $\H$ be a hypothesis class, and let $S$ be a non realizable sample of size $n >1$ 
such that every subsample of $S$ is realizable. 
Then, there exist two functions $F_a,F_b$
that map~$n$ bit-strings $x,y\in \{0,1\}^n$ to subsamples of $S$
such that $x\cap y = \emptyset$ if and only if
the joint sample $\concatxx{F_b(x)}{F_b(y)}$ is not realizable by $H$.
\end{lemma}

\begin{proof}[Proof of \cref{lem:conpreduction}]
Identify the domain of $S$ with $[n]$, and
write $S$ as $\big( (i,b_i) \big)_{i \in [n]}$.
Since $n>1$, for every $i$ there is a unique $b$ so that $(i,b)$ is in $S$.
For every $i$ so that $x_i=0$ put the example $(i,b_i)$ in $F_a$.
For every $i$ so that $y_i=0$ put the example $(i,b_i)$ in $F_b$.
If $i \in x \cap y$ then none of $(i,1),(i,-1)$ appear in $(F_a,F_b)$.
If $i \not \in x \cap y$ then $(i,b_i)$ appear in $\concatxx{F_a(x)}{F_b(y)}$.
In other words, $\concatxx{F_a}{F_b}= S$ if and only if $x \cap y = \emptyset$.

%
\end{proof}

We can now complete the proof of the theorem.
if $k=0$ there is nothing to prove, since $k=0$ if and only if all samples are realizable.
Let $S$ be a non realizable sample of size $k$
so that every subsample of $S$ is realizable;st
Let $F_a,F_b$ be as given by \cref{lem:conpreduction} for $S$.
The maps $F_a,F_b$ imply that if $\Tc(\H) \leq C$
then there is an NP-proof in Yao's model
for solving $\mathrm{DISJ}_k$ with bit-complexity $O(C \log k)$.
{This concludes the proof} since the non-deterministic communication complexity of 
$\mathsf{DISJ}_k$ is $\Omega(k)$  (by~\cref{thm:disj}).
%

\end{proof}
\subsection{Proof of \cref{thm:pchar} and \cref{thm:proper}}\label{prf:pchar} 
\pchar*

The crux of the proof is the following lemma, which yields protocol that decides
the realizability problem for $\H$ with sample complexity that efficiently depends
on the VC and coVC dimensions of $\H$.
\begin{lemma}\label{thm:upper}
For every class $\H$ with $\vc(\H)=d$ and $\covc(\H)=k$ there exists a protocol
for the realizability problem over $\H$ with sample complexity $O(dk^2{\log k} \log |S|)$.
\end{lemma}
Before giving a full detailed proof, we give a rough overview of the protocol, depicted in \cref{fig:alg}. In this protocol the players jointly run boosting over their samples.
All communication between parties is intended so that at iteration $t$, Alice and Bob agree upon a hypothesis $h_t$ which is \emph{simultaneously} an $\alpha$-weak hypothesis with
$\alpha = \frac{1}{2}- \frac{1}{5k}$
for Alice's distribution $p^a_t$ on $S_a$ and for Bob's distribution $p^b_t$ on $S_b$. 
The $\eps$-net theorem (\cref{thm:epsnet}) implies that to agree on such a hypothesis
Alice and Bob can each publish a subsample of size $O(d k {\log k})$;
every hypothesis that agrees with the published subsamples
has error at most $\frac{1}{5k}$ over both $p^a_t$ and $p^b_t$.
In particular, if no consistent hypothesis exists then the protocol terminates
with the output ``non-realizable''.
If the protocol has not terminated after $T = O( k\log |S|)$ rounds
then Alice and Bob output ``realizable''. 

The main challenge in the proof is showing that 
if the algorithm did not find a non realizable subsample
then the input is indeed realizable.
We begin by proving the following Lemma:

\begin{figure}
\begin{tcolorbox}\label{alg:main}
\begin{center}
{\bf A protocol for the realizability problem over $\H$}\\
\end{center}
\textbf{Input}: Samples $S_a,S_b$ from $\H$. \\ \ \\
\textbf{Protocol:}
\begin{itemize}
\item The player transmit $|S| = |S_a|+|S_b|$.
\item Let $p^a_1$ and $p^b_1$ to be uniform distributions over $S_a$ and $S_b$. 

\item If $S_a$ is not realizable then Alice returns $\textrm{NON-REALIZABLE}$, and
similarly if $S_b$ is not realizable then Bob returns $\textrm{NON-REALIZABLE}$.
\item For $t=1,\ldots, T= 4(k+1) \log\lvert S\rvert$
\begin{enumerate}
\item Alice sends a subsample $S'_a\subseteq S_a$ of size ${O}(dk{\log k})$ such that \emph{every} $h\in\H$ that is consistent with $S'_a$ has \[\sum_{z\in S_a} p_t^{a}(z) 1[h(x)\ne y] \leq \frac{1}{5k}.\]
Bob sends Alice a subsample $S'_b\subseteq S$
of size ${O}(dk{\log k})$ with the analogous property.
\item Alice and Bob check if there is $h\in H$ that is consistent with both $S'_a$ and $S'_b$.
If the answer is ``no'' then they return $\textrm{NON-REALIZABLE}$,
and else they pick $h_t$ to be such an hypothesis.
\item Bob and Alice both update their respective distributions as in boosting: 
Alice sets 
\[ p_{t+1}^a(z) \propto p_t^a 2^{-1[h(x)= y]} \quad \forall z\in S_a.\]
Bob acts similarly.
\end{enumerate}
\item If the protocol did not stop, then output REALIZABLE.
\end{itemize}
\end{tcolorbox}
\caption{A protocol for the realizability problem}\label{fig:alg}
\end{figure}

\begin{claim}\label{lem:nonrealize}
Let $\H$ be a class with coVC dimension $k > 0$.
For any unrealizable sample $S$ and for any $h_1,\ldots, h_T\in \H$,
there is $(x,y)\in S$ so that
\[ \frac{1}{T}\sum_{t=1}^T 1[h_t(x)\ne y] \ge \frac{1}{k}.\]
\end{claim}

\begin{proof}[Proof of \cref{lem:nonrealize}]\label{prf:nonrealize}
Let $S$ be an unrealizable sample. 
There exists a non realizable subsample $S'$ of $S$ of size at most $k$. 
Since $|S'|\le k$, for every hypothesis $h\in \H$ 
we have $L_{S'}(h) \geq 1/k$.
In particular, for a sequence $h_1,\ldots,h_T\subseteq \H$,
\begin{align*}
\max_{(x,y)\in S'} \frac{1}{T} \sum_{t=1}^T \err{h_t}& \ge 
\frac{1}{\lvert S'\rvert} \sum_{(x,y)\in S'} \frac{1}{T} \sum_{t=1}^T \err{h_t}\\
& \ge \frac{1}{T} \sum_{t=1}^T\frac{1}{\lvert S'\rvert} \sum_{(x,y)\in S'}  \err{h_t}
\geq \frac{1}{k}. 
\end{align*}
\end{proof}

We are now ready to prove \Cref{thm:upper}.

\begin{proof}[Proof of \Cref{thm:upper}]\label{prf:upper}
It is clear that if Alice or Bob declare \textrm{NON-REALIZABLE}, then indeed the sample is not realizable.
It remains to show that (i) they can always find the subsamples $S'_a$ and $S'_b$ with the desired property, (ii) if the protocol has not terminated after $T$ steps then the sample is indeed realizable.

Item (i) follows by plugging $\eps=1/(5k)$ in Theorem~\ref{thm:epsnet}.

To prove (ii), note that
$h_t$ is $\alpha$-weak for $\alpha = \frac{1}{2}-\frac{1}{5k}$
with respect to \emph{both} distributions $p_t^a$ and $p_t^b$.
Therefore, \cref{lem:supermajority} implies  
that if {$T \geq 4(k+1)\log \lvert S\rvert$ then 
\begin{equation}\label{eq:supermajority}
\forall(x,y)\in \concat : \frac{1}{T} \sum_{t=1}^T \err{h_t} < \frac{1}{2(k+1)},
\end{equation}
which} by \cref{lem:nonrealize} implies that ${\concat}$ is realizable by~$\H$.
 
\end{proof}

Finally, we can prove the theorem.
\begin{proof}[Proof of \cref{thm:pchar}] \

\noindent{$\bf\ref{it:p} \implies \ref{it:npconp} $.}
This implication is easy since the $\np$ and $\conp$
sample complexities lower bound the deterministic sample complexity.

\noindent{$\bf \ref{it:npconp} \implies \ref{it:vccovc}$.}
This implication is the content of \cref{thm:vcnp} and~\ref{thm:covconp}.
For example, if $\vc(\H) = \infty$
then $\Tnp(n)\geq\tilde\Omega(n)$ for every $n$.

\noindent{$\bf \ref{it:vccovc} \implies \ref{it:dk}$.} 
This implication is the content of \cref{thm:upper}.

\noindent{$\bf \ref{it:dk} \implies \ref{it:p}$.} 
By definition of $\p$.
\end{proof}

\paragraph{Proof of \cref{thm:proper}}\label{prf:proper}

\proper*
{
\begin{proof}
We start with the second part of the theorem.
Any proper learner can be used to decide the realizability problem,
by setting $\eps = \frac{1}{2n}$ and checking whether the output hypothesis is consistent with the sample
(this costs two more bits of communication).
Now, if $\H$ is not in $\p$
then one of $\vc(\H)$ or $\covc(\H)$ is infinite.
\cref{thm:vcnp} and \cref{thm:covconp}
imply that at least $\tilde \Omega(n)$ examples are needed to decide if a sample of size $n$ is realizable. 
So, at least $\tilde\Omega(1/\epsilon)$ examples are required in order to properly learn $\H$ in the realizable setting. 

It remains to prove the first part of the theorem.
Let $\H \in \p$.
Consider the following modification of the protocol in \cref{alg:main}:
\begin{itemize}
\item Each of Alice and Bob picks an $\epsilon$--net of their sample. 
Set $S'_{a}\subseteq S$ to be Alice's $\epsilon$-net and similarly Bob sets $S'_b\subseteq S$.
\cref{thm:epsnet} implies that $|S'_a|$ and $|S'_b|$ are at most
$s = O\bigl(\frac{d\log(1/\eps)}{\eps}\bigr)$.

\item Alice and Bob run the protocol in \cref{alg:main} with inputs $S_a',S_b'$.
Since the input is realizable, the protocol will complete all $T = \Theta  ( k \log  s )$ iterations.
In each iteration, $\tilde{O}(dk)$ samples are communicated.
The sample complexity is hence as claimed.

\item Set 
\[K =\left\{ x\in X : \text{all but a fraction of $< \frac{1}{2(k+1)}$ of the $h_t$'s agree on $x$}\right\}.\]
Alice and Bob output $h^*\in \H$ that agrees on $K$ with the majority of the $h_t$'s.
\end{itemize}

We next argue that this protocol succeeds.
First, observe that $K$ depends only on $h_1,\ldots, h_T$ and not on the sample points.
Hence both Alice and Bob have the necessary information to compute $K$.
Next, note that by \cref{eq:supermajority} for every $(x,y)\in \concatxx{S'_a}{S'_b}$, we have that $x\in K$ and that $y$ is the majority vote of the $h_t$'s on $x$. Thus, assuming that $h^*$ exists (which is established next), we have that $L_{S'_a}(h^*)=L_{S'_b}(h^*)=0$. As a corollary,
 its error on $S$ is at most $\eps$;
indeed,  $L_{S_a} (h^*) \le \epsilon$ since $S'_{a}$ is an $\epsilon$-net for $S_{a}$.
Similarly $L_{S_b}(h^*) \le \epsilon$. 
Overall it follows that \[L_{S}(h^*)\le \max\left\{ L_{S_a}(h^*), L_{S_b}(h^*)\right\} \le \epsilon.\]

It remains to show that $h^*$ exists.
We use Tychonoff's theorem from topology~\cite{Tychonoff30} to 
prove the following claim.

\medskip

\noindent {\em Claim.  Let $\H$ be a closed hypothesis class. Let $S$ be a (possibly infinite) set of labelled examples so that
for every finite subsample $S'$ of $S$
there is $h_{S'} \in \H$ that is consistent with $S'$,
then there is $h_S \in \H$ that is consistent with $S$.}

\begin{proof}
Recall that we call a class $\H$ closed if for every $g\not\in \H$ there exists a finite sample $S_g$ that is consistent with $g$ yet not realizable by $\H$. Our notion is consistent with the topological notion of a closed set if we identify $\H$ as a subset in $\{\pm 1\}^{\cal X}$ equipped with the product topology. To see that indeed $\H$ is (topologically) closed, note that for every $g \not \in \H$ there is a finite sample $S_g$
that is consistent with $g$ yet not realizable by $\H$.
Denote by $U_g$ the open subset of $\{\pm 1\}^{\cal X}$
of all functions that are consistent with $S_g$.
Thus for all $g \not \in H$ we have $H \cap U_g = \emptyset$.
So $\H$ is the complement of $\bigcup_g U_g$, which is open and thus closed in the topological sense. 
One can also verify that the converse holds, 
namely every topologically closed set $\H$ in $\{\pm 1\}^{X}$ induces a closed hypothesis class. 

Next, we employ Tychonoff's theorem that states that $\{\pm 1\}^{\cal X}$  is compact under  the product topology;
as a corollary we obtain that $\H$ is also compact.

Now, assume toward contradiction that there is no $h \in \H$ that is consistent with $S$.
For every finite subsample $S'$ of $S$, consider the closed subset $C_{S'}$ of 
$\{\pm 1\}^{{\cal X}}$ of all function that are consistent with $S'$.
Thus $\bigcap_{S'} C_{S'} = \emptyset$.
Since $\H$ is compact, this implies that there is a finite list $(S_\ell)_\ell$
of subsamples of $S$ so that $\bigcap_\ell C_{S_\ell} = \emptyset$.
This is a contradiction since we may unite $(S_\ell)_\ell$
to a single finite subsample of $S$, which is realizable by assumption.
\end{proof}

The existence of $h^*$ is now derived as follows.
Assume towards contradiction $h^*$ does not exist.
By the above claim, there is a finite sample $S_K$ of examples from $K$ labelled according to the majority of the $h_t$'s 
that is not realizable.
Now, by \cref{lem:nonrealize} there is $(x,y)\in S_K$
such that $\frac{1}{T}\sum_{t=1}^T 1[h_t(x)\ne y] \ge \frac{1}{k}$.
This implies that $(x,y)\notin S_K$,  a contradiction.
%
%

%

\end{proof}

\paragraph{Extensions of \cref{thm:proper} to non-closed classes.}
We do not know whether every (not necessarily closed)
class $\H$ can be learned properly in the realizable case with logarithmic communication complexity.
However, the closeness assumption can be replaced by other natural restrictions.
For example, consider the case where the domain $\cal X$ is countable,
and consider a class $\H$ that is in $\p$ (not necessarily closed).
We claim that in this case $\H$ can be properly learned with $\tilde O(\log 1/\eps)$
sample complexity:
The closure of $\H$, denoted by $\bar\H$ is obtained from $\H$
by adding to $\H$ all hypotheses $h$ such that every finite subsample
of $h$ is realizable by $\H$. Such a $h$ is called a \emph{limit-hypothesis}\footnote{This name is chosen to maintain consistency with the topological notion of a limit-point.} of $\H$.
Thus, by running the protocol from the above proof on $\bar \H$ 
(which has the same VC and coVC dimensions as $\H$),
Alice and Bob agree on a limit-hypothesis $h^*\in \bar \H$ with $\eps$ error.
The observation is that they can ``project'' $h^*$ back to $\H$
if they could both agree on a finite sample $S'$ that contains both input sample $S_a,S_b$.
Indeed, since $h^*$ is a limit-hypothesis of $\H$ then there is some $h\in H$
that agrees with $h^*$ on $S'$. 
Therefore, from the knowledge of $h^*$ and $S'$ 
Alice and Bob can output such an $h$ without further communication.
Thus, once having a limit-hypothesis $h^*$,
the problem of proper learning is reduced to finding a finite sample $S'$
that is consistent with $h^*$ and contains both input samples.
If $X$ is countable, say $X=\mathbb{N}$,
then Alice and Bob can simply transmit to each other
their two maximal examples to determine 
$x_{max} = \max\bigl\{x : (x,y)\in S_a\cup S_b)\bigr\}$,
and set $S' = \{\bigl(x, h^*(x)\bigr ) : x\leq x_{max}\}$.

A result from~\cite{bendavid17emx}
shows how to extend this scheme 
for any $\cal X$ such that $\lvert {\cal X} \rvert < \aleph_{\omega}$;
more specifically, if $\lvert {\cal X} \rvert = \aleph_k$ then Alice and Bob 
can agree on $S'$ with an additional cost of $O(k)$ examples.
To conclude, if $\lvert {\cal X} \rvert < \aleph_\omega$
then $\H$ is in $\p$ if and only if it can be properly learned
in the realizable setting with sample complexity $\tilde O(\log 1/\eps)$.
}


\subsection{Proof of \cref{thm:logn}}\label{prf:logn}

The statement clearly follows as a corollary of the following, stronger, statement which is a direct corollary of 
\cref{lem:cvxdisjoint}.

\begin{theorem}[Realizability problems -- lower bound (strong version)]\label{thm:lognstrong}
Let $\H$ be the class of half-planes in $\R^2$. Any communication protocol with the following properties must have sample complexity at least $\tilde\Omega(\log n)$
for samples of size $n$:
\begin{enumerate}[i]
\item Whenever the sample is realizable by $\H$ it outputs 1.
\item Whenever for some $x\in \R^2$, we have $\{(x,1),(x,-1)\}\subseteq  S$ it outputs 0.
\item[(iii)] It may output anything in the remaining case.
\end{enumerate}
\end{theorem}
\subsection{Proof of Theorem~\ref{thm:pcomp}}\label{prf:pcomp}
\pcomp*

\begin{proof} \

\noindent{$\bf \ref{it:1c} \implies \ref{it:2c}$.} 
By 
\cref{thm:pchar}, every $\H\in\p$ has a protocol of sample
complexity at most $c\cdot\log n$ for samples of size $n$ with $c=O(dk^2\log k)$
where $d = \vc(\H)$ and $k = \covc(\H)$.

\medskip

\noindent{$\bf \ref{it:2c} \implies \ref{it:3c}$.} 
Since the examples domain is restricted to $R$,
every protocol with sample complexity $T$
can be simulated by a protocol with bit complexity $O(T\log(\lvert R\rvert))$.
The input sample size $n$ can be assumed to be at most $2\lvert R\rvert$
(by removing repeated examples).

\medskip

\noindent{$\bf \ref{it:3c} \implies \ref{it:1c}$.}  
By Theorem~\ref{thm:pchar} it suffices to show that both the VC and the coVC dimensions of $\H$ are finite.
Indeed, let $m$ and $c$ be such that for any $R$ there exists a protocol in Yao's model for the realizability problem with complexity $c\cdot \log^m(|R|)$. If there is a shattered set $R$ of size $N$ 
then by \cref{thm:vcnp}:
\begin{align*}
c \log^m (N) \geq
N^{np}_{\H|_R}(N) \geq 
\tilde \Omega (N), 
\end{align*}
which shows that $N$ is bounded in terms of $c$ and $k$
{(the left inequality holds since the deterministic sample complexity in Yao's
model upper bounds the NP sample complexity in the model considered here).}
A similar bound on the coVC dimension follows by \cref{thm:covconp}.
\end{proof}

%% file: proofs-convex.tex
\starttocentries
\section{Lower Bound for Convex Set Disjointness}
\stoptocentries
\label{sec:cvxdisjoint}

We begin by stating a round elimination lemma in Yao's model. The proof of the round elimination lemma is given in \cref{sec:ProofRE}). 
We require the following additional notation:
for a function $D:X\times Y \to \{0,1\}$, and $m\in \mathbb{N}$, define 
a new function $D_m: (X^m) \times (Y\times [m]))\to \{0,1\}$ by
\begin{align*}
D_m((x_1,\ldots,x_m);(y,i))= D(x_i, y).
\end{align*}
Also, for a distribution $\mathbb{P}$ on $X\times Y$, let $\mathbb{P}_m$
be the distribution on $(X^m) \times (Y\times [m]))$ that is defined by the following sampling procedure:
\begin{itemize}
\item Sample $m$ independent copies $(x_j,y_j)$ from $\mathbb{P}$.
\item Sample $i\sim[m]$ uniformly and independently of previous choice.
\item Output $\bigl((x_1,\ldots,x_m) ; (y_i,i)\bigr)$.
\end{itemize}
\begin{lemma}[Round Elimination Lemma]\label{roundEliminationLemma} 
Let $D:X\times Y\to\{0,1\}$ be a function, let $m\in\mathbb{N}$, 
and let $\mathbb{P}$ be a distribution on $X\times Y$.
Assume there is a protocol in Yao's model for $D_m$,
where Alice's input is $\xx = (x_1,\ldots,x_m)$ and Bob's is $(y,i)$,
on inputs from $\mathbb{P}_m$ with error probability at most $\delta$
such that:
\begin{itemize}
\item Alice sends the first message.
\item It has at most $r$ rounds.
\item In each round at most $c$ bits are transmitted.
\end{itemize}
Then there is a protocol for $D$
that errs with probability at most $\delta+O\bigl((c/m)^{1/3}\bigr)$ 
on random inputs from $\mathbb{P}$ and satisfies:
\begin{itemize}
\item Bob sends the first message.
\item It has at most $r-1$ rounds.
\item In each round at most $c$ bits are transmitted.
\end{itemize}
\end{lemma}

The rest of the proof for the convex set disjointness lower bound is organized as follows.
We define the distribution over inputs to the convex set disjointness problem, 
denoted by $\sbd{m,r}=(A_{m,r},B_{m,r})$. 
%
%
%
Roughly, the key idea in the construction is that solving convex set disjointness with inputs $(A_{m,r},B_{m,r})$, where Alice's input is $A_{m,r}$ and Bob's input is $B_{m,r}$,
requires solving a function of the form $D_m$, where each $x_j$ is an independent instance of $B_{m,r-1}$ and each $y_j$ is an independent instance of $A_{m,r-1}$. 
This enables an inductive argument, using round elimination.
We {will then} conclude that Alice and Bob are unable to achieve probability error of less than $1/10$,  
unless a specified amount of bits is transmitted at each round. 


\subsubsection*{Construction of $\sbd{m,r}$}
\label{sec:specify}

Let $m\in\mathbb{N}$. 
For the base case, $r=1$, we set $\sbd{m,1}=(A_{m,0},B_{m,0})$, where
$A_{m,1} = \{(0,0)\}$ and $B_{m,1}$
is uniform on $\{(0,0),\emptyset\}$.
Define $\sbd{m,r}$ for $r>1$ inductively as follows:
\begin{itemize}
\item Let $p_1,\ldots,p_m$ be $m$ evenly spaced points on the positive part of the unit circle (i.e. the intersection of the unit circle with the positive cone $\{(x,y), x>0, y>0\}$);
\item Pick $\epsilon>0$ to be sufficiently small, as a function of $m,r$ (to be determined later).
\item For $1\leq i \leq m$ let $U_i:\R^2\to \R^2$ be the rotation matrix that transforms the $y$-axis to $p_i$ and the $x$-axis to $p_i^\perp$.
Define $T_i$ as the following affine transformation:
\begin{align*}
T_i(v) = U_i\left[ \begin{array}{cc} -\epsilon & 0 \\ 0 & -\epsilon^2\end{array}\right] v + p_i.
\end{align*}
From a geometric perspective $T_i$ acts on $v$ by rescaling $x$-distances by $\epsilon$ and $y$-distances by $\epsilon^2$,
reflecting through the origin, rotating by $U_i$ and then translating by $p_i$.
\end{itemize}
Define 
$$A_{m,r} = \bigcup_{j=1}^m T_j(\bir{j})$$ and $$B_{m,r}= T_i(\air{i}),$$
 where 
 $$(\air{1},\bir{1}),(\air{2},\bir{2}),\ldots,(\air{m},\bir{m})$$ are drawn i.i.d.\ 
 from $\sbd{m,r-1}$, and $i$ is uniform in $[m]$
 and independent of previous choices.
 Notice the compatibility with $\mathbb{P}_m$.

\subsection*{Properties of $\sbd{m,r}$}

Two crucial properties (which we prove below) of the distribution $\sbd{m,r}$ are given
by the following two lemmas.

\begin{lemma}\label{SmallUniverseLemma}
{There is a set $R_{m,r}\subseteq\R^2$ of size $\lvert R_{m,r}\rvert \le m^{r-1}$
such that each pair of sets in the support of $\sbd{m,r}$ is contained in 
$R_{m,r} \times R_{m,r}$.}
\end{lemma}

\begin{lemma}\label{IntersectionCriteriaLemma}
The following are equivalent (almost surely):
\begin{enumerate}
\item $\mathsf{conv}(A_{m,r})\cap\mathsf{conv}(B_{m,r}) = \emptyset$.  \label{it:conv1}
\item $\air{i} \cap \bir{i}= \emptyset$. \label{it:conv2}
\item $A_{m,r}\cap B_{m,r} = \emptyset$. \label{it:conv3}
\end{enumerate}
\end{lemma}


The first property implies that transmitting point from $A_{m,r}$ or $B_{m,r}$ in Yao's model
requires $r\log m$ bits. This allows us to translate lower bounds from Yao's to the model considered in this paper.
The second property is needed to apply the round elimination argument.
%

\cref{SmallUniverseLemma} follows by a simple induction on $r$.
The proof of the second lemma is more elaborate.
%

\ignore{
\begin{lemma}\label{IntersectionCriteriaLemma}
Consider $\sbd{m,r}$ a random sample. If $d>1$, let $A_j,B_j$ be the instances of $I_{b,d-1}$ and $i$ the index used to define $A,B$ as above. Let $\delta_d>0$ be sufficiently small given $p_i$ and the defining parameters of $I_{b,d-1}$. Let $N_d$ be the $\delta_d$-angle vertical cone with apex at $(0,0)$. Then the following are equivalent:
\begin{enumerate}
\item The convex hulls of $A$ and $B$ intersect.
\item The convex hulls of $B$ intersects the convex hull of $A\cup N_d$.
\item The convex hulls of $\air{i}$ and $\bir{i}$ intersect (for $i$ the specified index).
\item $A\cap B \neq \emptyset.$
\end{enumerate}
\end{lemma}
The third part of this equivalence will be critical in applying our round elimination lemma.
}
\begin{proof}[Proof of \cref{IntersectionCriteriaLemma}] \

\noindent{ \ref{it:conv1} $\implies$ \ref{it:conv2}} holds because $\air{i}\subseteq A_{m,r}$ and $\bir{i}\subseteq B_{m,r}$.

\medskip

\noindent{ \ref{it:conv2} $\implies$ \ref{it:conv3}} follows from the definition of $A_{m,r}$ and $B_{m,r}$
by setting $\eps$ sufficiently small so that the $m$ instantiations from $\sbd{m,r-1}$ are mutually disjoint.

\medskip 

\noindent{ \ref{it:conv3} $\implies$ \ref{it:conv1}} is the challenging direction, which we prove by induction on $r$. 
In order for the induction to carry, we slightly strengthen the statement and show that
if $A_{m,r}\cap B_{m,r} = \emptyset$ then they are separated by a vector $u\in \R_{+}^2$
with positive entries:
\begin{align*} 
\forall a\in A_{m,r} : u \cdot a < 1,\\
\forall b\in B_{m,r}: u \cdot b >1 ,
\end{align*}
where $\cdot$ is the standard inner product in $\R^2$.

The case of $r=1$ is trivial. 
Let $r>1$, and assume that $A_{m,r}\cap B_{m,r}=\emptyset$.
By the induction hypothesis, there is a vector $u=(\alpha,\beta)\in \R^2$ with $\alpha,\beta>0$
separating $\air{i}$ from $\bir{i}$.
We claim that the vector
$$u^* = \frac{1}{\beta-\eps^2} \tilde u$$
achieves the goal, where
\[\tilde{u} 
= U_i \left[ \begin{array}{cc} \eps & 0 \\ 0 & 1\end{array}\right] u
= U_i \left( \begin{array}{c} \eps \alpha  \\ \beta \end{array}\right) .\]

{First, we claim that $\tilde u$ can be written as}
\begin{equation}\label{eq:uclosetop}
\tilde u = \beta p_i + v_\eps ,
\end{equation}
where
$$
\|v_\eps \|_2 = \alpha \eps.$$
Indeed, recall that $U_i e_2=p_i$,
and so {$\tilde u = \beta p_i + \eps \alpha  U_i e_1$, and $\|\eps \alpha U_i e_1\|_2 = \eps \alpha$.}
Since $p_i$ has positive entries,
if $\eps$ is small enough we get that $\tilde u$ and $u^*$ have positive entries.

Next, we prove that
\begin{align*} 
 \forall a\in A_{m,r} :\tilde{u} \cdot a < \tilde u \cdot p_i -\eps^2,\\
\forall b\in B_{m,r} : \tilde{u} \cdot b >\tilde u \cdot  p_i -\eps^2
\end{align*}
The above completes the proof since $\tilde u \cdot p_i -\eps^2 = \beta-\eps^2$
and by choice of $u^*$.

Let $b\in B_{m,r}$ and $a \in \air{i}$ be so that $b=T_i(a)$.
Thus,
\begin{align*}
\tilde{u} \cdot b - \tilde u\cdot p_i &= \tilde{u} \cdot T_i(a) - \tilde u \cdot p_i\\ 
							&= \tilde u \cdot \Bigl(U_i\left[ \begin{array}{cc} -\epsilon & 0 \\ 0 & -\epsilon^2\end{array}\right] a\Bigr)
							\tag{by definition of $T_i$}\\
							&=\Bigl( U_i \left[ \begin{array}{cc} \eps & 0 \\ 0 & 1\end{array}\right] u
		\Bigr) \cdot
		\Bigl(U_i\left[ \begin{array}{cc} -\epsilon & 0 \\ 0 & -\epsilon^2\end{array}\right] a\Bigr)
							\tag{by the definition of $\tilde u$}\\
										&=\Bigl( \left[ \begin{array}{cc} -\epsilon & 0 \\ 0 & -\epsilon^2\end{array}\right]  \left[ \begin{array}{cc} \eps & 0 \\ 0 & 1\end{array}\right] u
		\Bigr) \cdot
		 a
							\tag{$U_i$ is orthogonal}\\
							&=-\eps^2 u \cdot a > -\eps^2.  \tag{by induction $u \cdot a < 1$}
\end{align*}
A similar calculation shows that $\tilde{u}^\top b - \tilde u^\top p_i < -\eps^2$ 
for $b\in T_i\bigl(\air{i}\bigr)$.

It remains to consider $a \in A_{m,r}$
and $b \in \bir{j}$ so that $a = T_j(b)$
for $j\neq i$:
\begin{align*}
\tilde{u} \cdot a - \tilde u \cdot p_i 
						&=\tilde u\cdot (p_j-p_i) + 
\tilde u\cdot \Big( U_j\left[ \begin{array}{cc} -\epsilon & 0 \\ 0 & -\epsilon^2\end{array}\right] b \Big)
\\
	&=\beta p_i \cdot (p_j-p_i) + e ,
\end{align*}
where
$$e = v_\eps \cdot (p_j-p_i) +
\tilde u\cdot \Big( U_j\left[ \begin{array}{cc} -\epsilon & 0 \\ 0 & -\epsilon^2\end{array}\right] b \Big)
.$$
Since $\beta p_i \cdot (p_j-p_i)<0$, and $\|e\|_2 \to 0$
when $\eps \to 0$,
picking a sufficiently small $\eps$ finishes the proof.

\end{proof}
\ignore{
\begin{proof}
We proceed by induction on $d$. For $d=1$, these claims are easy to verify. Hence, we can assume that these statements are equivalent when applied to $\air{i},\bir{i}$ instead of $A,B$.

It is clear that $4$ implies $3$ implies $1$ implies $2$, so we will try to prove that $2$ implies $4$. Firstly, note that the convex hulls of $A\cup N_d$ and $B$ intersect if and only if the convex hulls of $T_i^{-1}(A\cup N_d)$ and $T_i^{-1}(B)=\air{i}$ intersect. Note that $T_i^{-1}(A) = \bir{i} \bigcup_{j\neq i} T_i^{-1}(T_j(B_j)).$ However, it should be noted that since $T_i$ can be thought of as compressing $x$-distances by $\epsilon$ and $y$-distances by $\epsilon^2$ before applying a fixed linear transformation, if $\epsilon$ is taken to be small enough and $q$ is any point of the tangent line to $C$ at $p_i$, then $T_i^{-1}(q)$ is contained $N_{d-1}$. Thus, for $\epsilon$ small enough, $T_i^{-1}(A\cup N_d)$ is contained in $\bir{i}\cup N_{d-1}$. Therefore, the convex hulls of $A\cup N_d$ and $B$ intersect only if the convex hulls of $\bir{i}$ and $\air{i} \cup N_{d-1}$ intersect, which by the inductive hypothesis, happens if and only if the convex hulls of $\air{i}$ and $\bir{i}$ do. By the inductive hypothesis, this implies that $\air{i}\cap \bir{i} \neq \emptyset$, but if $x$ is in the intersection, then $T_i(x)$ is in the intersection of $A$ and $B$. This shows that $2$ implies $4$ and completes the proof.
\end{proof}
}
\subsection{Proof of Round Elimination Lemma}
\label{sec:ProofRE}

Here we prove \cref{roundEliminationLemma} using standard tools from information theory.
\begin{proof}
Let $\Pi_m$ be the assumed protocol for $D_m\bigl(\xx;(y,i)\bigr)$.
We use the following protocol for~$D(x,y)$: 
\begin{itemize}
\item Alice gets $x$ and Bob gets $y$.
\item 
Alice and Bob draw, using shared randomness, 
an index $i$
and independently Alice's first message $M$ in $\Pi_m$
(without any conditioning). 
\item
Alice draws inputs $x_1,\ldots, x_{i-1}, x_{i+1},\ldots, x_{m}$ conditioned on the value of $M$ and on $x_i=x$.
\item
Alice and Bob then run the remaining $r-1$ rounds of $\Pi_m$,
following the message $M$, on inputs 
$$\xx= (x_1,\ldots x_{i-1},x_i=x,x_{i+1},\ldots, x_{m})$$ and $(y,i)$.
\end{itemize}


The crucial observation is that if for the chosen $M,i$,
the variables $x,y$ are distributed like $\mathbb{P}_m(x_i,y_i|M,i)$, 
then the above protocol errs with probability at most~$\delta$,
since $D_m(\xx,(y,i))=D(x,y)$ and 
since $\delta$ is the error probability of $\Pi_m$.

It thus suffices to show that with probability at least $1-(c/m)^{1/3}$ over the choice of $(M,i)$, 
the distributions $\mathbb{P}_m(x_i,y_i|M,i)$ and $\mathbb{P}(x,y)$ are $O((c/m)^{1/3})$ close in 
total variation distance.

To prove this, we show that the mutual information between $(M,i)$ and
$(x_i,y_i)$ is small, and then use Pinsker's inequality to move to total variation distance.
{Since $x_1,\ldots,x_m$ are independent and $i$ is uniform,
$$
I(x_i;M|i) = 
\frac{1}{m} \sum_{j=1}^m I(x_j;M) \leq \frac{1}{m} I(x_1,\ldots,x_m;M) \leq \frac{c}{m} .
$$
Thus,
\begin{align*}
I((x_i,y_i);(M,i)) 
& = I(x_i;i)  + I(y_i;M,i|x_i) + I(x_i;M|i) 
 \leq 0 +  0 + \frac{c}{m}.
\end{align*}}
%


\ignore{
Suppose that after Alice's initial message, the value of the index $i$ were announced publicly. Alice and Bob are then trying to compute the value of $D(x_i,y_i)$. Thus, looking at the remainder of Alice and Bob's protocol would produce a protocol for $D$ with a minor adjustment. In particular, after conditioning on the message $M$ the distributions on $x_i$ and $x_i$ are not quite the standard distributions on $x$ and $y$ anymore. However, we will show that they are probably close.}

Write the mutual information in terms of KL-divergence,

since $(x,y)$ and $(x_i,y_i)$ have the same distribution,
$$
\mathbb{E}_{M,i}[ D_{KL}(p_{x_i,y_i|M,i} || p_{x,y}) ] = I(x_i,y_i;M,i) .
$$
By Markov's inequality, 
the probability over $M,i$ that 
$$D_{KL}(p_{x_i,y_i|M,i} || p_{x,y}) > (c/m)^{2/3}$$
is less than $(c/m)^{1/3}$.
Pinsker's inequality completes the proof.
\end{proof}

\subsection{Proof of \cref{lem:cvxdisjoint}}
\cvxdisjoint*

\begin{proof}
Choose $m=n^{1/r}$ (assume that $n$ is such that $m$ is an integer).
Consider the distribution $\sbd{m,r}$ on inputs for the convex set disjointness.
We reduce this problem to Yao's model. 
By Lemma \ref{SmallUniverseLemma}
and choice of $m$, any point can be transmitted in Yao's model
using at most $O(\log(n))$ bits. 

We will show that every protocol in Yao's model with $r$ rounds
and error probability at most $0.1$ must transmit $\tilde\Omega(n^{1/r})$ bits.
%
To do that, we would like to apply the round elimination lemma. 
 Recall that Alice's input $A_{m,r}$ is equivalent to being told $\bir{1},\ldots,\bir{m}$. 
 Similarly Bob's input amounts to $\air{i}$ and $i$. By \cref{IntersectionCriteriaLemma}, $\mathsf{conv}(A_{m,r})\cap\mathsf{conv}(B_{m,r})=\emptyset$  if and only if $\air{i}\cap\bir{i}=\emptyset$. Therefore, for $r>1$, deciding if $A_{m,r}$ and $B_{m,r}$ intersect or their convex hulls are disjoint is equivalent to solving the same problem with respect to $\sbd{m,r-1}$ when the roles of the players 
are switched.


Next, iterating Lemma \ref{roundEliminationLemma}, 
we have that if there is a protocol to solve $\sbd{m,r}$ in $r$ rounds with Alice speaking first, $c$ is the maximum number of bits of communication per round, and $0.1$ probability of error, 
then 
there is a protocol for $\sbd{m,1}$ 
with Alice speaking first and one round of communication and probability $0.1+O(r(c/m)^{1/3})$ of error. 
However, the error probability of every such protocol 
is at least $0.5$.
That is,
$0.1 + O(r(c/m)^{1/3}) \geq 0.5$, which implies
\[c \geq \Omega\Bigl(\frac{n^{1/r}}{r^3}\Bigr).\]

Going back to allowing the protocol to send points rather than bits.
If $k$ is the maximum number of points sent per round then
$k \geq\Omega\Bigl(\frac{n^{1/r}}{r^3\log n}\Bigr)$.
 Now, since in each round at least one point is being sent, 
 we get a lower bound of
 \begin{equation}\label{eq:tradeoff}
\Omega\Bigl(\frac{n^{1/r}}{r^3\log n} + r\Bigr)
\end{equation}
on the sample complexity of $r$-round protocols that achieve error at most $0.1$.
One can verify that 
$\frac{n^{1/r}}{r^3\log n} + r
=\tilde\Omega(\log n)$, as required.

\end{proof}
 
 \paragraph{Discussion.}
\cref{eq:tradeoff} yields a round-error tradeoff for learning half-planes.
Indeed, if $\Pi$ is an $r$-round protocol that learns half-planes in the realizable case with error $\eps$.
Then, by picking $n < 1/\eps$, it implies a protocol for convex set disjointness
with similar sample complexity (up to additive constants).
In particular, the sample complexity of such a protocol is bounded from below by
\[\Omega\Bigl(\frac{(1/\eps)^{1/r}}{r^3\log (1/\eps)} + r\Bigr) .\]
This matches a complementary upper bound given by~\cite{Balcan12dist} (Theorem 10 in their paper). 

\starttocentries

%% file: apndx-improper.tex
\section{A protocol for Improper Learning}\label{apx:improper}

In this section we review the basic technique to apply Adaboost in order to improperly learn a class $\H$. Similar protocols appear in \cite{Balcan12dist, Daume12efficient}.
\begin{theorem}[Improper Learning -- Realizable case]
Let $\H$ be a class with VC dimension $d$. There exists a protocol that learns $\H$, in the realizable setting, with sample complexity $O(d\log 1/\epsilon)$.
\end{theorem} 
\begin{proof}
We depict the protocol in \cref{fig:algimp}. In this protocol Alice and Bob, each run an Ada--Boost algorithm over their own private sample set privately. The only thing we make sure of is that at each round, Alice and Bob agree on a weak hypothesis $h_t$ that is simoultansouly a weak learner for Alice and Bob independently. The way we choose $h_t$ is by sharing a a set $S_a$ that is an $\frac{1}{4}$-net over Alice's set and a sample set $S_b$ that is a $\frac{1}{4}$-net over Bob's set. Alice and Bob then pick a hypothesis that is consistent with both $S_a$ and $S_b$. The sets $S_a$ and $S_b$ are of size $O(d)$. This guarnatees that at round $t$ we will have that $\E_{(x,y)\sim p_a^t} 1\left[h_t(x)\ne y\right]\le \frac{1}{4}$, similarly $h_t$ is a $1/4$-weak hypothesis w.r.t to $p_a^t$. 

Finally, \cref{thm:boosting} guarantees that after $O(\log 1/\epsilon)$ rounds, the majority of the hypotheses is consistent with at least $(1-\epsilon)|S_a|$ of Alices sample point and similarly with $(1-\epsilon)|S_b|$ of Bob's point. In particular the majority is consistent with $(1-\epsilon)|S|$ of the points in the entire sample set.
\end{proof}

\begin{figure}
\begin{tcolorbox}\label{alg:imp}
\begin{center}
{\bf A protocol for learning $\H$}\\
\end{center}
\textbf{Input}: Samples $S_a,S_b$ from $\H$. \\ \ \\
\textbf{Protocol:}
\begin{itemize}
\item Let $p^a_1$ and $p^b_1$ to be uniform distributions over $S_a$ and $S_b$. 

\item If $S_a$ is not realizable then Alice returns $\textrm{NON-REALIZABLE}$, and
similarly if $S_b$ is not realizable then Bob returns $\textrm{NON-REALIZABLE}$.
\item For $t=1,\ldots, T= O(d\log 1/\epsilon) \log\lvert S\rvert$
\begin{enumerate}
\item Alice sends a subsample $S'_a\subseteq S_a$ of size ${O}(d)$ such that \emph{every} $h\in\H$ that is consistent with $S'_a$ has \[\sum_{z\in S_a} p_t^{a}(z) 1[h(x)\ne y] \leq \frac{1}{4}.\]
Bob sends Alice a subsample $S'_b\subseteq S$
of size $O(d)$ with the analogous property.
\item Alice and Bob check if there is $h\in H$ that is consistent with both $S'_a$ and $S'_b$.
If the answer is ``no'' then they return $\textrm{NON-REALIZABLE}$,
and else they pick $h_t$ to be such an hypothesis.
\item Bob and Alice both update their respective distributions as in boosting: 
Alice sets 
\[ p_{t+1}^a(z) \propto p_t^a 2^{-1[h(x)= y]} \quad \forall z\in S_a.\]
Bob acts similarly.
\end{enumerate}
\item If the protocol did not stop, then output $\textrm{MAJORITY}(h_1,\ldots, h_T)$.
\end{itemize}
\end{tcolorbox}
\caption{A protocol for improper learning}\label{fig:algimp}
\end{figure}

%% file: main.bbl
\newcommand{\etalchar}[1]{$^{#1}$}
\begin{thebibliography}{MSWY15}

\bibitem[Abe80]{Abelson80}
Harold Abelson.
\newblock Lower bounds on information transfer in distributed computations.
\newblock {\em J. {ACM}}, 27(2):384--392, 1980.

\bibitem[ABM17]{ashtiani17agnostic}
Hassan Ashtiani, Shai Ben{-}David, and Abbas Mehrabian.
\newblock Agnostic distribution learning via compression.
\newblock {\em CoRR}, abs/1710.05209, 2017.

\bibitem[AD12]{AD11}
Alekh Agarwal and John~C. Duchi.
\newblock Distributed delayed stochastic optimization.
\newblock In {\em Proceedings of the 51th {IEEE} Conference on Decision and
  Control, {CDC} 2012, December 10-13, 2012, Maui, HI, {USA}}, pages
  5451--5452, 2012.

\bibitem[AS15]{AS15}
Yossi Arjevani and Ohad Shamir.
\newblock Communication complexity of distributed convex learning and
  optimization.
\newblock In {\em Advances in Neural Information Processing Systems 28: Annual
  Conference on Neural Information Processing Systems 2015, December 7-12,
  2015, Montreal, Quebec, Canada}, pages 1756--1764, 2015.

\bibitem[AUY83]{Aho83notions}
Alfred~V. Aho, Jeffrey~D. Ullman, and Mihalis Yannakakis.
\newblock On notions of information transfer in {VLSI} circuits.
\newblock In {\em Proceedings of the 15th Annual {ACM} Symposium on Theory of
  Computing, 25-27 April, 1983, Boston, Massachusetts, {USA}}, pages 133--139,
  1983.

\bibitem[BBFM12]{Balcan12dist}
Maria{-}Florina Balcan, Avrim Blum, Shai Fine, and Yishay Mansour.
\newblock Distributed learning, communication complexity and privacy.
\newblock In {\em {COLT} 2012 - The 25th Annual Conference on Learning Theory,
  June 25-27, 2012, Edinburgh, Scotland}, pages 26.1--26.22, 2012.

\bibitem[BDM{\etalchar{+}}14]{balcan14learning}
Maria{-}Florina Balcan, Amit Daniely, Ruta Mehta, Ruth Urner, and Vijay~V.
  Vazirani.
\newblock Learning economic parameters from revealed preferences.
\newblock {\em CoRR}, abs/1407.7937, 2014.

\bibitem[BHM{\etalchar{+}}17]{bendavid17emx}
Shai Ben{-}David, Pavel Hrubes, Shay Moran, Amir Shpilka, and Amir Yehudayoff.
\newblock A learning problem that is independent of the set theory {ZFC}
  axioms.
\newblock {\em CoRR}, abs/1711.05195, 2017.

\bibitem[BL98]{DL98}
Shai Ben{-}David and Ami Litman.
\newblock Combinatorial variability of vapnik-chervonenkis classes with
  applications to sample compression schemes.
\newblock {\em Discrete Applied Mathematics}, 86(1):3--25, 1998.

\bibitem[BU16]{bendavid16version}
Shai Ben{-}David and Ruth Urner.
\newblock On version space compression.
\newblock In {\em Algorithmic Learning Theory - 27th International Conference,
  {ALT} 2016, Bari, Italy, October 19-21, 2016, Proceedings}, pages 50--64,
  2016.

\bibitem[CBC16]{Chen16boosting}
Shang{-}Tse Chen, Maria{-}Florina Balcan, and Duen~Horng Chau.
\newblock Communication efficient distributed agnostic boosting.
\newblock In {\em Proceedings of the 19th International Conference on
  Artificial Intelligence and Statistics, {AISTATS} 2016, Cadiz, Spain, May
  9-11, 2016}, pages 1299--1307, 2016.

\bibitem[DGSX12]{DGSX12}
Ofer Dekel, Ran Gilad{-}Bachrach, Ohad Shamir, and Lin Xiao.
\newblock Optimal distributed online prediction using mini-batches.
\newblock {\em Journal of Machine Learning Research}, 13:165--202, 2012.

\bibitem[DMY16]{david16statistical}
Ofir David, Shay Moran, and Amir Yehudayoff.
\newblock {On statistical learning through the lens of compression}.
\newblock In {\em NIPS}, 2016.

\bibitem[Fre95]{freund95}
Yoav Freund.
\newblock Boosting a weak learning algorithm by majority.
\newblock {\em Inf. Comput.}, 121(2):256--285, 1995.

\bibitem[FS97]{freund97decision}
Yoav Freund and Robert~E. Schapire.
\newblock A decision-theoretic generalization of on-line learning and an
  application to boosting.
\newblock {\em J. Comput. Syst. Sci.}, 55(1):119--139, 1997.

\bibitem[FW95]{FW95}
Sally Floyd and Manfred~K. Warmuth.
\newblock Sample compression, learnability, and the vapnik-chervonenkis
  dimension.
\newblock {\em Machine Learning}, 21(3):269--304, 1995.

\bibitem[GKN14]{gottlieb14near}
Lee{-}Ad Gottlieb, Aryeh Kontorovich, and Pinhas Nisnevitch.
\newblock Near-optimal sample compression for nearest neighbors.
\newblock In {\em Advances in Neural Information Processing Systems 27: Annual
  Conference on Neural Information Processing Systems 2014, December 8-13 2014,
  Montreal, Quebec, Canada}, pages 370--378, 2014.

\bibitem[HW87]{Haussler87epsilon}
David Haussler and Emo Welzl.
\newblock $\eps$-nets and simplex range queries.
\newblock {\em Discrete {\&} Computational Geometry}, 2:127--151, 1987.

\bibitem[HY14]{Huang14epsapprox}
Zengfeng Huang and Ke~Yi.
\newblock The communication complexity of distributed epsilon-approximations.
\newblock In {\em 55th {IEEE} Annual Symposium on Foundations of Computer
  Science, {FOCS} 2014, Philadelphia, PA, USA, October 18-21, 2014}, pages
  591--600, 2014.

\bibitem[IPSV12a]{Daume12efficient}
Hal~Daum{\'{e}} III, Jeff~M. Phillips, Avishek Saha, and Suresh
  Venkatasubramanian.
\newblock Efficient protocols for distributed classification and optimization.
\newblock In {\em Algorithmic Learning Theory - 23rd International Conference,
  {ALT} 2012, Lyon, France, October 29-31, 2012. Proceedings}, pages 154--168,
  2012.

\bibitem[IPSV12b]{Daume12protocols}
Hal~Daum{\'{e}} III, Jeff~M. Phillips, Avishek Saha, and Suresh
  Venkatasubramanian.
\newblock Protocols for learning classifiers on distributed data.
\newblock In {\em Proceedings of the Fifteenth International Conference on
  Artificial Intelligence and Statistics, {AISTATS} 2012, La Palma, Canary
  Islands, April 21-23, 2012}, pages 282--290, 2012.

\bibitem[KL18]{kothari}
Pravesh~K. Kothari and Roi Livni.
\newblock Agnostic learning by refuting.
\newblock {\em CoRR}, abs/1709.03871, 2018.
\newblock To appear at ITCS'18.

\bibitem[KN97]{Kushilevitz97book}
Eyal Kushilevitz and Noam Nisan.
\newblock {\em Communication complexity}.
\newblock Cambridge University Press, 1997.

\bibitem[KS92]{Kalyanasundaram92disj}
Bala Kalyanasundaram and Georg Schnitger.
\newblock The probabilistic communication complexity of set intersection.
\newblock {\em {SIAM} J. Discrete Math.}, 5(4):545--557, 1992.

\bibitem[KSW17]{kontorovich17nearest}
Aryeh Kontorovich, Sivan Sabato, and Roi Weiss.
\newblock Nearest-neighbor sample compression: Efficiency, consistency,
  infinite dimensions.
\newblock {\em CoRR}, abs/1705.08184, 2017.

\bibitem[LS93]{lovasz93communication}
L{\'a}szl{\'o} Lov{\u a}sz and Michael Saks.
\newblock Communication complexity and combinatorial lattice theory.
\newblock {\em Journal of Computer and System Sciences}, 47(2):322 -- 349,
  1993.

\bibitem[LS13]{LS13}
Roi Livni and Pierre Simon.
\newblock Honest compressions and their application to compression schemes.
\newblock In {\em {COLT} 2013 - The 26th Annual Conference on Learning Theory,
  June 12-14, 2013, Princeton University, NJ, {USA}}, pages 77--92, 2013.

\bibitem[LW86]{Littlestone86relating}
N.~Littlestone and M.~Warmuth.
\newblock Relating data compression and learnability.
\newblock {\em Unpublished}, 1986.

\bibitem[MSWY15]{MSWY15}
Shay Moran, Amir Shpilka, Avi Wigderson, and Amir Yehudayoff.
\newblock Compressing and teaching for low vc-dimension.
\newblock In {\em {IEEE} 56th Annual Symposium on Foundations of Computer
  Science, {FOCS} 2015, Berkeley, CA, USA, 17-20 October, 2015}, pages 40--51,
  2015.

\bibitem[MWW93]{Matousek93discrepancy}
Ji{\v{r}}{\'{\i}} Matou{\v{s}}ek, Emo Welzl, and Lorenz Wernisch.
\newblock Discrepancy and approximations for bounded vc-dimension.
\newblock {\em Combinatorica}, 13(4):455--466, 1993.

\bibitem[MY16]{MY15}
Shay Moran and Amir Yehudayoff.
\newblock Sample compression schemes for {VC} classes.
\newblock {\em J. {ACM}}, 63(3):21:1--21:10, 2016.

\bibitem[NW93]{nisan93rounds}
Noam Nisan and Avi Wigderson.
\newblock Rounds in communication complexity revisited.
\newblock {\em {SIAM} J. Comput.}, 22(1):211--219, 1993.

\bibitem[PS84]{papa84comm}
Christos~H. Papadimitriou and Michael Sipser.
\newblock Communication complexity.
\newblock {\em J. Comput. Syst. Sci.}, 28(2):260--269, 1984.

\bibitem[Raz92]{Razborov92disj}
Alexander~A. Razborov.
\newblock On the distributional complexity of disjointness.
\newblock {\em Theor. Comput. Sci.}, 106(2):385--390, 1992.

\bibitem[SF12]{schapire2012boosting}
Robert~E Schapire and Yoav Freund.
\newblock {\em Boosting: Foundations and algorithms}.
\newblock MIT press, 2012.

\bibitem[SFBL97]{schapire97boosting}
Robert~E. Schapire, Yoav Freund, Peter Barlett, and Wee~Sun Lee.
\newblock Boosting the margin: {A} new explanation for the effectiveness of
  voting methods.
\newblock In {\em Proceedings of the Fourteenth International Conference on
  Machine Learning {(ICML} 1997), Nashville, Tennessee, USA, July 8-12, 1997},
  pages 322--330, 1997.

\bibitem[SS14]{SS14}
Ohad Shamir and Nathan Srebro.
\newblock Distributed stochastic optimization and learning.
\newblock In {\em 52nd Annual Allerton Conference on Communication, Control,
  and Computing, Allerton 2014, Allerton Park {\&} Retreat Center, Monticello,
  IL, September 30 - October 3, 2014}, pages 850--857, 2014.

\bibitem[SSBD14]{shalev14understanding}
Shai Shalev-Shwartz and Shai Ben-David.
\newblock {\em Understanding Machine Learning: From Theory to Algorithms}.
\newblock Cambridge University Press, New York, NY, USA, 2014.

\bibitem[SSZ14]{SST14}
Ohad Shamir, Nathan Srebro, and Tong Zhang.
\newblock Communication-efficient distributed optimization using an approximate
  newton-type method.
\newblock In {\em Proceedings of the 31th International Conference on Machine
  Learning, {ICML} 2014, Beijing, China, 21-26 June 2014}, pages 1000--1008,
  2014.

\bibitem[Tyc30]{Tychonoff30}
A.~Tychonoff.
\newblock {\"U}ber die topologische erweiterung von r{\"a}umen.
\newblock {\em Mathematische Annalen}, 102(1):544--561, Dec 1930.

\bibitem[Vad17]{vadhan}
Salil~P. Vadhan.
\newblock On learning vs. refutation.
\newblock In {\em Proceedings of the 30th Conference on Learning Theory, {COLT}
  2017, Amsterdam, The Netherlands, 7-10 July 2017}, pages 1835--1848, 2017.

\bibitem[VC71]{Vapnik69uniform}
V.N. {Vapnik} and A.Ya. {Chervonenkis}.
\newblock {On the uniform convergence of relative frequencies of events to
  their probabilities.}
\newblock {\em {Theory Probab. Appl.}}, 16:264--280, 1971.

\bibitem[WE15]{wiener15agnostic}
Yair Wiener and Ran El{-}Yaniv.
\newblock Agnostic pointwise-competitive selective classification.
\newblock {\em J. Artif. Intell. Res.}, 52:171--201, 2015.

\bibitem[Yao79]{Yao79}
Andrew Chi-Chih Yao.
\newblock Some complexity questions related to distributive computing
  (preliminary report).
\newblock In {\em STOC}, pages 209--213, 1979.

\end{thebibliography}
